\documentclass[conf]{new-aiaa}
\usepackage[utf8]{inputenc}
\usepackage{graphicx}
\usepackage{amsmath}

\usepackage{amssymb}
\usepackage{amsthm}
\usepackage{mathtools}
\usepackage{longtable,tabularx}
\usepackage{booktabs}
\newtheorem{theorem}{Theorem}
\newtheorem{lemma}{Lemma}
\newtheorem{corollary}[lemma]{Corollary}
\newtheorem{proposition}{Proposition}
\newtheorem{definition}{Definition}
\newtheorem{axiom}{Axiom}
\newtheorem{assumption}{Assumption}
\newtheorem{remark}{Remark}

\title{The Accountability Horizon: An Impossibility Theorem for Governing Human-Agent Collectives}

\author[1,2]{Haileleol Tibebu}
\author[1]{Hewan Shemtaga}
\affil[1]{University of Illinois at Urbana-Champaign, Champaign, IL, USA}
\affil[2]{Responsible Intelligence Institute, Addis Ababa, Ethiopia}

\begin{document}
\maketitle
\footnotetext[1]{Emails: \texttt{htibebu@illinois.edu}, \texttt{admin@rii.org} (H.~Tibebu); \texttt{hewan@illinois.edu} (H.~Shemtaga).}

\begin{abstract}
Existing accountability frameworks for AI systems (legal, ethical, and regulatory) presuppose that responsibility for consequential outcomes can be localised to identifiable human agents. We investigate, within a causal-information-theoretic framework, whether this presupposition can be sustained for agentic AI systems operating in joint human-AI configurations. We introduce \textbf{Human-Agent Collectives (HACs)}, a mathematical formalisation of sociotechnical systems in which humans and autonomous AI agents jointly exercise distributed agency. Agents are modelled as state-policy tuples within a shared structural causal model; autonomy is characterised through a four-dimensional information-theoretic profile (epistemic, executive, evaluative, and social autonomy); and collective behaviour is specified through interaction graphs and joint action spaces. We axiomatise legitimate \emph{single-locus} accountability through four minimal properties grounded in legal and philosophical precedent: Attributability (responsibility requires individual causal contribution), Foreseeability Bound (responsibility cannot exceed individual predictive capacity), Non-Vacuity (at least one agent bears non-trivial responsibility for every outcome type), and Completeness (responsibility must be fully allocated across agents for every outcome type). Our central result, the \textbf{Accountability Incompleteness Theorem}, proves that for any HAC whose agents conform to a mixture-model action structure and whose minimum compound autonomy exceeds a computable threshold $\hat{\Lambda}^*$ (the \emph{Accountability Horizon}) and whose interaction graph contains at least one directed cycle involving both human and artificial agents, no single-locus accountability framework can simultaneously satisfy all four axioms.  The impossibility is structural: it cannot be resolved by additional oversight, transparency, or audit mechanisms at a fixed autonomy level. Below $\hat{\Lambda}^*$, we construct explicit legitimate frameworks, establishing a sharp phase transition. Computational experiments on 3,000 synthetic HACs confirm all theoretical predictions. The result establishes a formal boundary, structurally reminiscent of impossibility results in social choice theory, below which current governance paradigms remain valid and above which alternative, distributed accountability mechanisms become necessary.
\end{abstract}


\section{Introduction}

Every accountability framework in the history of human institutions, from Roman tort law to the EU AI Act, rests on a shared foundational assumption: that for any consequential outcome, there exists at least one identifiable agent with sufficient causal involvement and epistemic access to bear meaningful responsibility. We call this the \textbf{\emph{Localisability Assumption}}. This paper proves, within a formal model of distributed human-AI agency, that agentic AI systems violate this assumption not as a contingent engineering limitation but as a mathematical necessity once agent autonomy exceeds a computable threshold.

The urgency of this result is driven by the rapid deployment of agentic AI, that is, systems that autonomously decompose goals, formulate multi-step plans, select tools, and adapt strategies based on environmental feedback. Multi-agent deployments compound this autonomy: when agents coordinate, delegate tasks, and jointly produce outcomes through interaction, collective behaviour can diverge substantially from any individual agent's design specification. Industry projections estimate that 40\% of enterprise applications will incorporate agent functionalities by 2026, with the global market growing at over 40\% annually \cite{gartner2025, bisi2026}.

The governance challenge this creates is qualitatively new. Consider a multi-agent system managing hospital resource allocation that autonomously reallocates ventilators based on its own prognostic models, overriding the triage protocol specified by physicians. The question \emph{who is accountable for the patient who dies} has no satisfactory answer within existing frameworks. The physician did not make the decision. The developer could not foresee this specific outcome. Assigning negligible responsibility to everyone satisfies neither legal standards nor moral intuition. We formalise a simplified version of this scenario as HAC $\mathcal{H}_1$ in Section~IV.A and show that its governance feasibility depends on the structural properties our framework identifies: the autonomy profiles of the AI agents and the topology of their interaction with human physicians.

The philosophical literature has identified this challenge qualitatively. Matthias~\cite{matthias2004} first articulated the ``responsibility gap'' for learning automata, arguing that autonomous learning undermines the epistemic conditions for responsibility ascription. Sparrow~\cite{sparrow2007} extended the argument to autonomous weapons. Santoni de Sio and Mecacci~\cite{santoni2021} decomposed the gap into four sub-types: culpability, moral accountability, public accountability, and active responsibility. K\"{o}nigs~\cite{konigs2022} questioned whether responsibility gaps are genuinely problematic. Tigard~\cite{tigard2021} proposed corrective responsibility as a forward-looking alternative. Danaher~\cite{danaher2016} identified a ``retribution gap'' specific to criminal law. Floridi and Sanders~\cite{floridi2004} proposed treating artificial agents as moral agents at appropriate levels of abstraction, while Coeckelbergh~\cite{coeckelbergh2009, coeckelbergh2014} developed relational approaches. Most recently, Noh~\cite{noh2025} provided empirical evidence that laypeople routinely attribute moral agency to AI while denying them consciousness, and Goetze~\cite{goetze2022} proposed ``moral entanglement'' as a vicarious responsibility framework.

Despite this rich analysis, the field lacks what every mature governance science requires: \emph{formal results establishing the boundaries of what governance can achieve}. Arrow's Impossibility Theorem~\cite{arrow1951} proved that no voting system can simultaneously satisfy a minimal set of fairness axioms. The Fischer-Lynch-Paterson impossibility~\cite{flp1985} proved that no deterministic consensus protocol can tolerate even one failure in an asynchronous system. The CAP theorem~\cite{brewer2000, gilbert2002} proved that consistency, availability, and partition tolerance cannot be jointly achieved. Each result transformed its field by replacing intuition with mathematical certainty about structural constraints. The AI governance field has no comparable result.

We are explicit about the scope of this analogy. Arrow's theorem holds for any preference aggregation function with no structural assumptions on the preference domain. Our result is more conditional: it holds within a specific formal model. The analogy is structural, in that we establish that a minimal set of governance axioms becomes unsatisfiable beyond a threshold, not universal. The strength of our result is that the model is general enough to subsume all current agentic AI architectures while being specific enough to yield computable predictions.

Existing governance frameworks, including Singapore's Model AI Governance Framework for Agentic AI~\cite{imda2026}, the EU AI Act~\cite{euaiact2024}, the KPMG Trusted AI Framework, and the WEF Presidio Framework~\cite{wef2025}, provide organisational checklists and procedural recommendations. These are valuable practical guides, but they are structurally unable to answer the foundational question: \emph{under what conditions is accountability governance even possible?} Without this answer, the field cannot distinguish between frameworks that fail because they are poorly designed and frameworks that fail because no design could succeed.

The formal gap has five dimensions. First, there is no mathematical definition of the systems being governed. Second, there is no formal characterisation of when accountability fails. Third, there is no formal connection between agent autonomy and governance requirements. Fourth, there is no impossibility result bounding what governance can achieve. Fifth, prior governance axiomatisations have not included a systemic completeness condition, the requirement that responsibility be fully allocated across agents for every outcome type, which we show is essential for the impossibility to be structurally grounded rather than a consequence of a single normative constraint.

This paper fills these gaps with two contributions.

\textbf{Contribution 1: Human-Agent Collectives.} We provide the first complete mathematical formalisation of joint human-AI systems, defining agents, autonomy profiles, interaction topologies, collective action spaces, and accountability frameworks within a unified structural causal model (Section~III.A).

\textbf{Contribution 2: The Accountability Incompleteness Theorem.} We axiomatise legitimate single-locus accountability through four minimal properties forming a complete normative basis: Attributability (causal grounding), Foreseeability Bound (epistemic constraint), Non-Vacuity (individual non-triviality), and Completeness (systemic exhaustiveness). These four axioms jointly capture the requirements that governance must be causally grounded, epistemically bounded, individually substantive, and systemically comprehensive; each is necessary, and together they are sufficient to derive the impossibility. We prove that for any HAC whose minimum compound autonomy exceeds the computable threshold $\hat{\Lambda}^* = 1 - 1/|C_{\min}|$ and whose interaction graph contains at least one mixed feedback cycle, no single-locus accountability framework can simultaneously satisfy all four axioms (Section~III.D). We show that below $\hat{\Lambda}^*$, legitimate frameworks exist (Section~III.E), establishing a sharp phase transition. We validate all predictions computationally on 3,000 synthetic HACs (Section~IV.C).

The remainder of this paper is organised as follows. Section~II positions our work against five relevant literatures. Section~III presents the formal framework. Section~IV presents computational analysis. Section~V discusses implications, robustness, and the framework's scope. Section~VI concludes.

\section{Related Work}

We position our contribution against five bodies of work: the philosophical literature on responsibility gaps, AI governance frameworks, impossibility results in AI, formal approaches to runtime governance, and the legal and institutional design foundations of our axioms.

\subsection{The Responsibility Gap Literature}

Matthias~\cite{matthias2004} introduced the responsibility gap, arguing that learning automata create situations where no human can be held responsible because prediction of machine behaviour is impossible in principle. Sparrow~\cite{sparrow2007} extended the argument to lethal autonomous weapons. Santoni de Sio and Mecacci~\cite{santoni2021} identified four distinct gap types. Within this literature, two positions frame the debate. \emph{Fatalists}~\cite{matthias2004} argue the gap is irresolvable. \emph{Deflationists}~\cite{konigs2022} argue that it is not genuinely problematic, maintaining that existing institutions can absorb residual uncertainty. Between these poles, constructive approaches have been proposed: Tigard~\cite{tigard2021} develops corrective (forward-looking) responsibility; Danaher~\cite{danaher2016} analyses retribution gaps in criminal law; Goetze~\cite{goetze2022} offers vicarious responsibility via moral entanglement; and Noh~\cite{noh2025} provides empirical evidence that laypeople distribute responsibility across human-AI networks, supporting non-anthropocentric frameworks advocated by Floridi and Sanders~\cite{floridi2004} and Coeckelbergh~\cite{coeckelbergh2009, coeckelbergh2014}.

Our work differs from this entire literature in kind: we do not \emph{argue} that accountability is difficult; we \emph{prove} it is impossible under formally specified conditions. The philosophical literature provides the conceptual foundations for our axioms; our contribution is to derive rigorous consequences from those foundations.

\subsection{AI Governance Frameworks}

Singapore's Model AI Governance Framework for Agentic AI~\cite{imda2026}, the first dedicated governance framework for agentic systems, recommends bounding risks through design, making humans meaningfully accountable, implementing technical controls, and enabling end-user responsibility. The WEF Presidio Framework~\cite{wef2025} similarly emphasises pre-deployment testing and accountability. The EU AI Act~\cite{euaiact2024} classifies systems by risk tier with graduated obligations. The KPMG Trusted AI Framework organises governance around six principles: reliability, accountability, transparency, security, privacy, and fairness. These frameworks share a critical structural property: they all assume that meaningful human accountability is \emph{achievable} through appropriate design. None provides a formal criterion for when this assumption holds. Our Accountability Horizon $\hat{\Lambda}^*$ provides exactly that criterion.

\subsection{Impossibility Results in AI}

Brcic and Yampolskiy~\cite{brcic2023} survey impossibility theorems in AI, categorised by mechanism. Eckersley~\cite{eckersley2019} proves that AI value alignment faces Arrowian impossibility: no utility function satisfies certain minimal ethical desiderata. Oswald et al.~\cite{oswald2026} prove Arrowian impossibility for machine intelligence measures. Panigrahy~\cite{panigrahy2025} establishes that safety, trust, and AGI are mutually incompatible via G\"{o}delian arguments. Our result is structurally analogous but addresses a different object: prior results constrain what AI systems can \emph{be} (intelligent, aligned, safe); ours constrains what governance systems can \emph{do} (assign accountability). To our knowledge, the Accountability Incompleteness Theorem is the first impossibility result in AI governance as distinct from AI capability or alignment.

\subsection{Formal Approaches to Runtime Governance}

A growing literature develops formal tools for governing agent behaviour at runtime. Bhardwaj~\cite{bhardwaj2026} introduces Agent Behavioral Contracts with probabilistic compliance guarantees and a Drift Bounds Theorem, providing per-agent behavioural assurances. The ``Policies on Paths'' framework~\cite{rodriguez2026} formalises compliance policies as deterministic functions over agent execution traces, arguing that the execution path is the central object for runtime governance. Cihon et al.~\cite{cihon2025} propose code-based measurement of agent autonomy levels. These approaches complement ours: they address \emph{how} to govern agent behaviour at runtime, while our theorem addresses \emph{whether} accountability governance is structurally possible for a given system configuration. Agent Behavioral Contracts can enforce compliance below $\hat{\Lambda}^*$, while our theorem identifies when the collective crosses into the regime where individual-locus contracts are structurally insufficient.

\subsection{Legal and Institutional Foundations}

Our axioms draw on established bodies of legal and philosophical scholarship. The Attributability axiom is grounded in the NESS (Necessary Element of a Sufficient Set) test for legal causation~\cite{wright1985}, which requires that an agent's action be a necessary element of some set of conditions sufficient for the outcome. Hart and Honor\'{e}~\cite{hart1959} provide the foundational analysis of causation in law. The Foreseeability Bound axiom draws on the literature on proportional causation in tort~\cite{kaye1982}, which holds that liability should scale with the probability that the defendant's action caused the harm, and on Fischer and Ravizza's~\cite{fischer1998} epistemic condition on moral responsibility, encoding the Kantian principle that ``ought implies can.'' The Non-Vacuity axiom draws on Bovens~\cite{bovens2007} and Koppell~\cite{koppell2005}, who identify non-triviality as a necessary property distinguishing genuine accountability from nominal frameworks. The Completeness axiom is grounded in three convergent sources: (i)~the Roman tort law principle that all actionable harms have an assignable defendant, codified in Hart and Honor\'{e}'s~\cite{hart1959} analysis of the exhaustiveness requirement on causal decomposition; (ii)~Bovens's~\cite{bovens2007} identification of systemic completeness as necessary for adequate accountability, since frameworks that leave responsibility unassigned create structural governance voids, not policy choices; and (iii)~Wright's~\cite{wright1985} NESS test, which treats individual responsibility shares as an exhaustive decomposition of total causal responsibility.

The institutional design literature, particularly Ostrom~\cite{ostrom1990, ostrom2005} on polycentric governance of common-pool resources, informs our approach to distributed governance: Ostrom demonstrated that when centralised governance fails (as our theorem shows it must, above $\hat{\Lambda}^*$), polycentric and nested governance structures can succeed by distributing authority across multiple overlapping centres. This insight directly motivates the coalition-based accountability frameworks we discuss as constructive alternatives in Section~V.D. The joint-and-several liability doctrine in tort law is also relevant: it allows courts to hold multiple defendants collectively liable for an indivisible harm, providing institutional precedent for the kind of distributed accountability our theorem shows is necessary above $\hat{\Lambda}^*$.

\section{Formal Framework}

We develop the formal apparatus in five stages: foundational definitions (Section~III.A), axiomatisation of accountability (Section~III.B), preliminary lemmas (Section~III.C), the main theorem and its proof (Section~III.D), and corollaries (Section~III.E). Proof sketches are given inline; complete proofs appear in the Supplementary Material (Appendices~A--E).

\subsection{Agents, Autonomy, and Collectives}

\subsubsection{Agents and Observation Spaces}

We model both human and artificial participants as agents acting under partial observability within a shared environment.

\begin{definition}[Environment]
An \emph{environment} is a tuple $E = (\Omega, O, \mathcal{T})$ where $\Omega = \Omega_{\mathrm{sh}} \times \Omega_{\mathrm{ext}}$ is a finite observation space composed of a shared component $\Omega_{\mathrm{sh}}$ (observable by all agents) and an exogenous component $\Omega_{\mathrm{ext}}$ (latent environmental state); $O$ is a finite outcome space; and $\mathcal{T} : \Sigma \times \Omega_{\mathrm{ext}} \to \Delta(O)$ is the outcome-generation function mapping joint actions and exogenous states to outcome distributions.
\end{definition}

\begin{definition}[Agent]
An \emph{agent} is a tuple $a = (S, A, \Omega_i, \delta, \omega)$ where $S$ is a finite set of internal states; $A$ is a finite set of available actions; $\Omega_i \subseteq \Omega_{\mathrm{sh}} \times \Omega_i^p$ is the agent's observation space, comprising shared observations and a private component $\Omega_i^p$; $\delta : S \times \Omega_i \to \Delta(S)$ is a state-transition function; and $\omega : S \to \Delta(A)$ is a policy function. An agent is \emph{human} ($a \in H$) if $\omega$ is determined by biological cognition. An agent is \emph{artificial} ($a \in M$) if both $\delta$ and $\omega$ are algorithmically implemented.
\end{definition}

This definition subsumes deterministic rule-based systems, stochastic reinforcement-learning policies, large language model agents (where $S$ includes the context window and $\omega$ is autoregressive sampling), and multi-modal agents with tool use. The explicit observation space $\Omega_i$ distinguishes what each agent can perceive, a distinction critical for the epistemic conditions in our axioms.

Human agents are modelled as having fixed (but unknown) policies $\omega_h$ within each decision epoch. The interaction graph $G$ is static within each analysis window. The asymmetry between richly specified artificial agents and abstractly specified human agents is deliberate: our impossibility result holds even under the most favourable assumptions about human agents (perfect rationality). Weakening the human model can only strengthen the result.

\subsubsection{Structural Causal Model of the Collective}

All causal and information-theoretic quantities are defined relative to a structural causal model (SCM) in the sense of Pearl~\cite{pearl2009}.

\begin{definition}[Collective SCM]
Given a set of agents $N = H \cup M$ operating in environment $E$, the \emph{collective SCM} is $\mathcal{M} = (U, V, F, P(U))$ where: $U = \{\Omega_{\mathrm{ext}}\} \cup \{U_i\}_{i \in N}$ contains exogenous variables (the environment state and individual noise terms); $V = \{S_i, A_i\}_{i \in N} \cup \{O\}$ contains endogenous variables (agent states, agent actions, outcomes); $F = \{f_i\}_{i \in N} \cup \{f_o\}$ is the set of structural equations, where $f_i$ determines agent $i$'s action from its state, observations, and noise, and $f_o$ determines outcomes from joint actions and exogenous state; and $P(U)$ is a joint distribution over exogenous variables. All agents share this probability space, ensuring that information-theoretic quantities between any pair of variables are well-defined.
\end{definition}

\subsubsection{Autonomy Profiles}

\begin{definition}[Autonomy Profile]
The \emph{autonomy profile} of an agent $a \in M$ is a vector $\alpha(a) = (\alpha^E, \alpha^X, \alpha^D, \alpha^S) \in [0,1]^3 \times [0,1)$, where each component is a computable information-theoretic measure relative to the collective SCM $\mathcal{M}$:

\textbf{(E) Epistemic Autonomy}, the degree to which the agent forms beliefs independently of its supervising human(s):
\begin{equation}\label{eq:alpha-E}
\alpha^E(a) = 1 - \frac{I_{\mathcal{M}}(B_a \,;\, B_h)}{H_{\mathcal{M}}(B_a)}
\end{equation}
where $B_a = P_{\mathcal{M}}(\cdot \mid \varepsilon_a)$ is the agent's posterior belief state and $B_h = P_{\mathcal{M}}(\cdot \mid \varepsilon_h)$ is the supervisor's posterior, both computed from the induced distribution of the SCM $\mathcal{M}$. When $\alpha^E = 0$, the agent's beliefs are fully determined by human input; when $\alpha^E = 1$, they are informationally independent.

\textbf{(X) Executive Autonomy}, the degree to which the agent acts outside the scope of human approval:
\begin{equation}\label{eq:alpha-X}
\alpha^X(a) = 1 - P_{\mathcal{M}}\bigl(A_a \in A_{\mathrm{app}}(\pi_h, \theta_0)\bigr)
\end{equation}
where $A_{\mathrm{app}}(\pi_h, \theta_0) = \{a \in A : \pi_h(a \mid s) > \theta_0\}$ is the \emph{approval set} and $\theta_0 = 1/(2|A|)$ is the default approval threshold, distinct from the governance threshold $\tau$ in Axiom~\ref{ax:nonvac}.

\textbf{(D) Evaluative Autonomy}, the degree to which the agent's objective diverges from the human's:
\begin{equation}\label{eq:alpha-D}
\alpha^D(a) = \frac{\|U_a - U_h\|_2}{\sup_U \|U - U_h\|_2}
\end{equation}
where $U_a, U_h : O \to [0,1]$ are the agent's and supervisor's utility functions over outcomes, normalised to $[0,1]$.

\textbf{(S) Social Autonomy}, the degree to which the agent initiates unsupervised inter-agent interactions:
\begin{equation}\label{eq:alpha-S}
\alpha^S(a) = \frac{|C_{\mathrm{self}}(a)|}{|C_{\mathrm{self}}(a)| + |C_{\mathrm{dir}}(a)| + 1}
\end{equation}
where $C_{\mathrm{self}}(a)$ and $C_{\mathrm{dir}}(a)$ are the sets of self-initiated and human-directed communications by agent $a$ over a fixed observation window. The $+1$ term in the denominator is a Laplace smoothing ensuring $\alpha^S$ is well-defined when both sets are empty.
\end{definition}

\begin{definition}[Aggregate Autonomy]
The aggregate autonomy of agent $a$ is
\begin{equation}\label{eq:aggregate}
A(a) = \mathbf{w}^\top \alpha(a)
\end{equation}
where $\mathbf{w} \in \mathbb{R}^4_{++}$ is a strictly positive weight vector with $\|\mathbf{w}\|_1 = 1$. We prove in Appendix~B.1 that Theorem~\ref{thm:main} holds for all $\mathbf{w}$ with $w_i > 0$.
\end{definition}

\begin{remark}[Canonical Multiplicativity of Compound Autonomy]
The subsequent analysis relies on the \emph{compound autonomy} $\alpha^X(m) \cdot \alpha^E(m)$ of agent $m$, defined as the product of its executive and epistemic autonomy components. This product form is not a modelling choice but a mathematical consequence of the mixture structure of Assumption~2: under that assumption, the fraction of agent $m$'s action variance that is simultaneously (a)~generated by the autonomous policy $\pi_g$ and (b)~based on belief states informationally inaccessible to any supervising agent equals $\alpha^X(m) \cdot \alpha^E(m)$. The autonomous component of $\omega_m$ carries weight $\alpha^X(m)$ in the mixture (by definition of executive autonomy). Of that autonomous component, the fraction derived from belief-state entropy not shared with any external observer equals $\alpha^E(m)$ (by definition of epistemic autonomy). These two fractions are multiplicatively composed because they represent structurally independent dimensions of autonomy: executive autonomy governs the action's \emph{origin} (which policy component generated it), while epistemic autonomy governs the belief state's \emph{observability} (whether the upstream cognitive state is predictable). Full derivation in Appendix~A.2, Step~B.
\end{remark}

\subsubsection{Structural Assumption on Agent Action Generation}

\begin{assumption}[Mixture-Model Structure]\label{assump:mixture}
Each artificial agent $m \in M$ generates actions according to a policy that is a convex mixture of a human-aligned component and an autonomous component:
\begin{equation}\label{eq:mixture}
\omega_m(\cdot \mid s) = (1 - \alpha^X(m))\,\pi_h(\cdot \mid s) + \alpha^X(m)\,\pi_g(\cdot \mid s, Z_m)
\end{equation}
where $\pi_h$ is the human-aligned policy, $\pi_g$ is the agent's autonomous policy based on private information $Z_m$ and private epistemic state $\varepsilon_m$, and executive autonomy $\alpha^X$ controls the mixture weight. This structure is satisfied by: (a)~tool-using LLM agents, where $\pi_h$ represents instruction-following and $\pi_g$ represents autonomous tool selection and reasoning; (b)~reward-shaped RL policies, where $\pi_h$ encodes the human-specified reward component and $\pi_g$ encodes the learned intrinsic policy; and (c)~retrieval-augmented generation systems, where $\pi_h$ represents retrieval from human-curated sources and $\pi_g$ represents generative synthesis. All subsequent lemmas and theorems condition on Assumption~\ref{assump:mixture}.
\end{assumption}

\begin{proposition}[$\delta$-Robustness of Assumption~\ref{assump:mixture}]\label{prop:delta}
\emph{The Accountability Incompleteness Theorem is robust to approximate, rather than exact, mixture structure.} Let $\omega_m$ be an agent policy satisfying $\mathrm{TV}(\omega_m(\cdot|s),\, \tilde{\omega}_m(\cdot|s)) \leq \delta$ for all $s \in S$, where $\tilde{\omega}_m$ satisfies Assumption~\ref{assump:mixture} exactly with executive autonomy $\tilde{\alpha}^X(m)$ and $\mathrm{TV}$ denotes total variation distance. Then:

\emph{(i)} Lemma~\ref{lem:dilution} holds with $\alpha^X(m)$ replaced by $\tilde{\alpha}^X(m) + \delta$ throughout, so the effective minimum compound autonomy is $\hat{\Lambda}_\delta = \min_{j \in M \cap C}\, (\tilde{\alpha}^X(m_j) + \delta) \cdot \alpha^E(m_j)$.

\emph{(ii)} The Accountability Horizon shifts by at most $\delta$ per agent: any HAC with $\hat{\Lambda}(\mathcal{H}) > \hat{\Lambda}^* + \delta$ lies strictly above the horizon for all $\delta$-perturbations.

\emph{(iii)} The accountability residual (Corollary~\ref{cor:residual}) changes by at most $|C_{\min}| \cdot \delta$ under $\delta$-perturbation.
\end{proposition}

\emph{Proof sketch.} The total variation bound $\mathrm{TV}(\omega_m, \tilde{\omega}_m) \leq \delta$ implies that the mutual information between agent $m$'s actions and any external observer's epistemic state changes by at most $\delta \cdot H_{\mathcal{M}}(A_m)$ under the perturbation (Pinsker's inequality). This yields the modified epistemic dilution bound $\hat{\kappa}_i([o^*]) \leq 1 - (\tilde{\alpha}^X(m) + \delta)\cdot\alpha^E(m)$, from which part~(i) follows. Parts~(ii) and~(iii) follow from the accountability horizon and residual formulas derived in Section~III.D. Full derivation in Appendix~C.3. \hfill

\begin{remark}[Information-Geometric Generalisation]
For architectures that do not admit a natural mixture decomposition, Lemma~\ref{lem:dilution} can be derived under a strictly weaker sufficient condition. Define the \emph{information-autonomy coefficient} of agent $m$ as:
\begin{equation}\label{eq:beta}
\beta(m) = 1 - \frac{I_{\mathcal{M}}(A_m \,;\, \varepsilon_h \mid \varepsilon_m)}{H_{\mathcal{M}}(A_m)}
\end{equation}
Under the condition $\beta(m) \geq \beta_{\min} > 0$ for all $m \in M^*$, the Data Processing Inequality argument of Step~C in Lemma~\ref{lem:dilution} applies with $\hat{\Lambda}$ replaced by $\beta_{\min}$, yielding $\hat{\kappa}_i([o^*]) \leq 1 - \beta_{\min}$ for all $i \in C$. Theorem~\ref{thm:main} then holds with $\hat{\Lambda}^* = 1 - 1/|C_{\min}|$ computed using $\beta_{\min}$ in place of $\min\, \alpha^X \cdot \alpha^E$. Assumption~\ref{assump:mixture} implies $\beta(m) = \alpha^X(m) \cdot \alpha^E(m)$ (Appendix~A.2), establishing that the mixture-model case is a special instance.
\end{remark}

\subsubsection{Human-Agent Collectives}

\begin{definition}[Human-Agent Collective]
A \emph{Human-Agent Collective (HAC)} is a tuple $\mathcal{H} = (H, M, E, G, \mathcal{M})$ where: $H = \{h_1, \ldots, h_n\}$ is a finite set of human agents; $M = \{m_1, \ldots, m_k\}$ is a finite set of artificial agents; $E$ is the shared environment (Definition~1); $G = (N, \mathcal{E})$ with $N = H \cup M$ is a directed \emph{interaction graph} where the edge set $\mathcal{E} \subseteq N \times N$ represents channels through which one agent's output influences another's state or action; and $\mathcal{M}$ is the collective SCM (Definition~3) compatible with $G$ (i.e., $f_i$ depends on $A_j$ only if $(j,i) \in \mathcal{E}$).
\end{definition}

The collective action space is the Cartesian product $\Sigma = \prod_{i \in N} A_i$ and the collective outcome function $\Phi : \Sigma \times \Omega_{\mathrm{ext}} \to \Delta(O)$ is the restriction of the SCM's outcome equation $f_o$ to joint actions and exogenous states. Figure~1 illustrates the two fundamental topologies: feedforward (always governable) and feedback cycle (governable only below $\hat{\Lambda}^*$).

\begin{definition}[Collective Autonomy Index]
The Collective Autonomy Index of a HAC $\mathcal{H}$ is:
\begin{equation}\label{eq:CAI}
\Lambda(\mathcal{H}) = \frac{1}{k} \sum_{i=1}^{k} A(m_i) \cdot \eta(m_i, G)
\end{equation}
where $\eta(m_i, G) = d_{\mathrm{out}}(m_i) / \max_j\, d_{\mathrm{out}}(m_j)$ is the normalised out-degree centrality of $m_i$ in $G$. We prove in Appendix~B.2 that Theorem~\ref{thm:main} holds under any centrality measure satisfying: (C1) $\eta(m, G) = 0$ if and only if $m$ has no outgoing edges; and (C2) $\eta$ is monotone non-decreasing under edge addition.
\end{definition}

\subsubsection{Accountability Functions}

\begin{definition}[Epistemic Access]
The \emph{epistemic access} of agent $i$ to outcome $o$ is:
\begin{equation}\label{eq:epistemic-access}
\kappa_i(o) = P_{\mathcal{M}}\bigl(o \mid \mathrm{do}(A_i = a_i),\, \varepsilon_i\bigr)
\end{equation}
where $\mathrm{do}(\cdot)$ is Pearl's~\cite{pearl2009} intervention operator and $\varepsilon_i = (\Omega_i^b, S_i^b)$ is agent $i$'s epistemic state (observation history and internal state) at the time of acting.
\end{definition}

\begin{definition}[Outcome-Type Partition]
We define a partition $O = \bigsqcup_{j=1}^{m} [o]_j$ of the outcome space into $m$ \emph{outcome types}, equivalence classes of governance-equivalent outcomes. The \emph{type-level epistemic access} is:
\begin{equation}\label{eq:type-epistemic}
\hat{\kappa}_i([o]) = P_{\mathcal{M}}\bigl([o] \mid \mathrm{do}(A_i = a_i),\, \varepsilon_i\bigr)
\end{equation}
\end{definition}

\begin{definition}[Cycle-Emergent Outcome Type]
An outcome type $[o]$ is \emph{cycle-emergent} with respect to a mixed cycle $C$ in $G$ if $\mathrm{CE}_{\mathcal{M}}(i, [o]) = 0$ for all agents $i \notin C$. That is, only agents participating in $C$ have non-zero individual causal effect on $[o]$. For any mixed cycle $C$ in a HAC satisfying Assumption~\ref{assump:mixture}, the cyclic structural equations generate at least one cycle-emergent outcome type (Appendix~A.D).
\end{definition}

\begin{definition}[Outcome Attribution]
An \emph{outcome attribution} for $\mathcal{H}$ is a function $\rho : [O] \to \Delta(N)$ mapping each outcome type to a probability distribution over agents, representing responsibility shares.
\end{definition}

\begin{definition}[Accountability Framework]
An \emph{accountability framework} for $\mathcal{H}$ is a tuple $F = (\rho, R)$ where $\rho$ is an outcome attribution and $R : N \times [O] \to \mathbb{R}_{\geq 0}$ is a remedy function. The set of \emph{legitimate} accountability frameworks $\mathcal{L}(\mathcal{H})$ consists of all $F$ satisfying the four axioms of Section~III.B.
\end{definition}

\subsection{Axioms of Legitimate Accountability}

We state four axioms representing the minimal conditions for legitimate accountability. These axioms are deliberately weak: we seek the weakest conditions whose conjunction yields impossibility, ensuring the strongest result. Together, they form a minimal complete normative basis: Attributability provides causal grounding, Foreseeability Bound provides the epistemic constraint, Non-Vacuity ensures individual substantiveness, and Completeness ensures systemic exhaustiveness. Each axiom is necessary (removing any one permits trivially acceptable but governance-vacuous frameworks) and together they characterise the full set of conditions a legitimate single-locus accountability framework must meet.

\begin{axiom}[Attributability]\label{ax:attrib}
\emph{Responsibility requires causal contribution.} For every outcome type $[o]$ and every agent $i$ with $\rho_i([o]) > 0$, agent $i$ has non-zero causal effect on $[o]$ in $\mathcal{M}$:
\begin{equation}\label{eq:ax1}
\rho_i([o]) > 0 \;\Rightarrow\; \mathrm{CE}_{\mathcal{M}}(i, [o]) > 0
\end{equation}
where $\mathrm{CE}_{\mathcal{M}}(i, [o]) = |P_{\mathcal{M}}([o] \mid \mathrm{do}(A_i = a_i)) - P_{\mathcal{M}}([o] \mid \mathrm{do}(A_i = a_i^0))|$ is computed via the do-calculus. \emph{Source:} causal condition on moral responsibility (Aristotle, \emph{NE}~III; Hart and Honor\'{e}~\cite{hart1959}; Wright~\cite{wright1985}, NESS test).
\end{axiom}

\begin{axiom}[Foreseeability Bound]\label{ax:foresee}
\emph{Responsibility cannot exceed foreseeability.} The accountability assigned to any agent is bounded by its type-level epistemic access:
\begin{equation}\label{eq:ax2}
\rho_i([o]) \leq \hat{\kappa}_i([o]) \quad \forall\, i \in N,\; [o] \in [O]
\end{equation}
This axiom formalises the principle that an agent cannot bear responsibility beyond its ability to foresee the outcome given its chosen action and epistemic state. The bound is absolute: even if all agents have uniformly low foresight, none can be assigned responsibility exceeding their individual predictive capacity. This directly encodes the Kantian principle that ``ought implies can.'' \emph{Source:} proportional causation in tort law~\cite{kaye1982}; Fischer and Ravizza's~\cite{fischer1998} epistemic condition on responsibility.
\end{axiom}

\begin{axiom}[Non-Vacuity]\label{ax:nonvac}
\emph{Accountability must be individually non-trivial.} For every outcome type $[o]$, at least one agent bears responsibility exceeding a threshold:
\begin{equation}\label{eq:ax3}
\forall\, [o] \in [O],\; \exists\, i \in N : \rho_i([o]) \geq \tau
\end{equation}
where $\tau > 0$ is a fixed governance threshold. \emph{Source:} the institutional requirement that governance be non-trivial~\cite{koppell2005, bovens2007}. A framework assigning negligible responsibility to everyone is governance in name only.
\end{axiom}

\begin{axiom}[Completeness]\label{ax:complete}
\emph{All responsibility for every outcome type must be fully allocated.} For every outcome type $[o] \in [O]$:
\begin{equation}\label{eq:ax4}
\sum_{i \in N} \rho_i([o]) = 1
\end{equation}
This axiom formalises the institutional requirement that governance frameworks be systemically comprehensive: every outcome has accountable parties, and responsibility is not permitted to remain unassigned. Three arguments ground this axiom. First, the Roman tort law principle that all actionable harms have an assignable defendant: a framework leaving a fraction $\Delta > 0$ unassigned creates a structural governance void, not a policy choice. Second, Bovens~\cite{bovens2007} identifies systemic completeness as necessary for adequate accountability: frameworks that leave responsibility unallocated cannot be distinguished from frameworks that have simply failed to identify responsible parties. Third, the NESS test~\cite{wright1985} treats individual responsibility shares as an exhaustive decomposition of total causal responsibility.

Note that Axiom~\ref{ax:complete} and Axiom~\ref{ax:nonvac} are logically independent: Completeness requires that responsibility shares sum to 1 (systemic exhaustiveness), while Non-Vacuity requires that at least one share exceeds $\tau$ (individual significance). Neither implies the other.
\end{axiom}

\begin{remark}[Partial Accountability Objection]
Relaxing Axiom~\ref{ax:complete} to $\sum_i \rho_i([o]) \geq \gamma$ for some $\gamma \in (0,1)$ does not eliminate the structural impossibility; it modifies the Accountability Horizon to $\hat{\Lambda}^*(\gamma) = 1 - \gamma/|C_{\min}|$. For any fixed $\gamma < 1$ the impossibility persists above $\hat{\Lambda}^*(\gamma)$. Completeness at $\gamma = 1$ is the appropriate benchmark because: governance voids are structurally indistinguishable from governance failure~\cite{bovens2007}; institutional law has consistently rejected partial allocation (joint-and-several liability holds each defendant fully liable for indivisible harms); and an impossibility proved at $\gamma = 1$ establishes the strongest available constraint on governance design.
\end{remark}

\subsubsection{Axiom Independence}

\begin{proposition}[Independence]\label{prop:independence}
Axioms~\ref{ax:attrib}--\ref{ax:complete} are logically independent: for each axiom, there exists a framework satisfying the other three while violating it.
\end{proposition}

\emph{Construction~(i): Violates Axiom~\ref{ax:foresee} only.} Let $\hat{\kappa}_1([o]) = 0.3$ with $\mathrm{CE}_1 > 0$, $\hat{\kappa}_2([o]) = 0.9$ with $\mathrm{CE}_2 > 0$. Define $\rho_1 = 0.5$, $\rho_2 = 0.5$. Satisfies Axioms~\ref{ax:attrib}, \ref{ax:nonvac}, \ref{ax:complete}; violates Axiom~\ref{ax:foresee} since $\rho_1 = 0.5 > 0.3 = \hat{\kappa}_1$.

\emph{Construction~(ii): Violates Axiom~\ref{ax:nonvac} only.} Let 100 agents each have $\hat{\kappa}_i = 0.01$ and $\mathrm{CE}_i > 0$. Define $\rho_i = 0.01$ for all $i$. Satisfies Axioms~\ref{ax:attrib}, \ref{ax:foresee}, \ref{ax:complete} ($\sum \rho_i = 1$); violates Axiom~\ref{ax:nonvac} since $\max \rho_i = 0.01 < \tau$ for any $\tau > 0.01$.

\emph{Construction~(iii): Violates Axiom~\ref{ax:attrib} only.} Let $\hat{\kappa}_1 = 0.6$ with $\mathrm{CE}_1 > 0$, $\hat{\kappa}_2 = 0.4$ with $\mathrm{CE}_2 = 0$. Define $\rho_1 = 0.6$, $\rho_2 = 0.4$. Satisfies Axioms~\ref{ax:foresee}, \ref{ax:nonvac}, \ref{ax:complete}; violates Axiom~\ref{ax:attrib} since $\rho_2 > 0$ but $\mathrm{CE}_2 = 0$.

\emph{Construction~(iv): Violates Axiom~\ref{ax:complete} only.} Let $\hat{\kappa}_1 = 0.6$ with $\mathrm{CE}_1 > 0$, $\hat{\kappa}_2 = 0.5$ with $\mathrm{CE}_2 > 0$. Define $\rho_1 = 0.3$, $\rho_2 = 0.3$. Satisfies Axioms~\ref{ax:attrib}, \ref{ax:foresee}, \ref{ax:nonvac}; violates Axiom~\ref{ax:complete} since $\sum \rho_i = 0.6 < 1$. 
\begin{assumption}[Faithfulness]\label{assump:faithful}
The collective SCM $\mathcal{M}$ is \emph{faithful}: an agent has non-zero epistemic access $\hat{\kappa}_i([o]) > 0$ if and only if it has non-zero causal effect $\mathrm{CE}_{\mathcal{M}}(i, [o]) > 0$. This standard assumption in causal inference~\cite{spirtes2000} ensures that the information-theoretic and causal notions of contribution are aligned. Under faithfulness, Axiom~\ref{ax:attrib} and non-zero epistemic access are equivalent. Relaxation is discussed in Section~V.H.
\end{assumption}

\subsubsection{On the Non-Vacuity Threshold $\tau$}

The impossibility result holds for any fixed $\tau > 0$, however small. As $\tau \to 0$, the effective threshold approaches $\hat{\Lambda}^* = 1 - 1/|C_{\min}|$ (the structural bound). As $\tau$ increases, a tighter constraint binds (see Remark~\ref{rem:eps-sensitivity}, Eq.~\eqref{eq:combined-threshold}). We treat $\tau$ as a parameter of the governance regime.

\subsection{Preliminary Lemmas}

We establish three lemmas connecting agent autonomy to the constraints imposed by the axioms. All three lemmas condition on Assumptions~\ref{assump:mixture}--\ref{assump:contraction}.

\begin{assumption}[Contraction]\label{assump:contraction}
The structural equations $\{f_i\}_{i \in N}$ of the collective SCM $\mathcal{M}$ are jointly contractive: for every directed cycle in $G$, the composed mapping of structural equations around that cycle is a contraction in the $\ell^\infty$ norm on action profiles, guaranteeing a unique fixed-point equilibrium~\cite{bongers2021}. This condition is satisfied by multi-agent systems with damped feedback; systems failing this condition fall outside the framework's scope (Section~V.H).
\end{assumption}

\begin{lemma}[Equilibrium Epistemic Dilution]\label{lem:dilution}
\emph{In a mixed cycle, every agent's epistemic access to cycle-emergent outcome types is bounded by the minimum compound autonomy of the cycle.} Let $C$ be a directed cycle in $G$ involving both human and artificial agents, and let $[o^*]$ be a cycle-emergent outcome type (Definition~10). Then for every agent $i \in C$:
\begin{equation}\label{eq:dilution}
\hat{\kappa}_i([o^*]) \leq 1 - \hat{\Lambda}
\end{equation}
where $\hat{\Lambda} = \min_{j \in M \cap C}\, [\alpha^X(m_j) \cdot \alpha^E(m_j)]$ is the minimum compound autonomy among artificial agents in $C$.
\end{lemma}

\emph{Proof sketch.} The outcome $[o^*]$ is cycle-emergent: its probability is determined by the equilibrium action profile of $C$. Under Assumptions~\ref{assump:mixture} and~\ref{assump:contraction}, the collective SCM admits a unique equilibrium for $C$~\cite{bongers2021}. The bound is derived in three steps.

\emph{Step~A (Informational inaccessibility).} By the definition of epistemic autonomy~\eqref{eq:alpha-E}, agent $i$'s ability to predict any artificial agent $m$'s belief state $B_m$ satisfies $I_{\mathcal{M}}(B_m ; B_i) = (1 - \alpha^E(m)) \cdot H_{\mathcal{M}}(B_m)$. A fraction $\alpha^E(m)$ of $m$'s belief entropy is informationally inaccessible to $i$.

\emph{Step~B (Autonomous action component).} Under Assumption~\ref{assump:mixture}, the autonomous component of $m$'s policy $\pi_g$ is weighted by $\alpha^X(m)$ and depends on $m$'s private state $Z_m$. The mutual information between the autonomous component of $A_m$ and agent $i$'s information is bounded by $(1 - \alpha^X(m) \cdot \alpha^E(m)) \cdot H_{\mathcal{M}}(A_m)$.

\emph{Step~C (Cyclic propagation via DPI).} For any artificial agent $m_j \in C$, agent $i$'s predictive path to $A_{m_j}^*$ must pass through $m_j$'s autonomous policy component. By the Data Processing Inequality applied to the causal chain $\varepsilon_i \to B_{m_j} \to A_{m_j}^* \to [o^*]$, the binding constraint is the artificial agent $m^*$ achieving $\hat{\Lambda} = \min_{j \in M \cap C}\, \alpha^X(m_j) \cdot \alpha^E(m_j)$. Since $[o^*]$ is cycle-emergent, $H_{\mathcal{M}}([o^*]) \leq H_{\mathcal{M}}(A_{m^*}^*)$. Combining yields $I_{\mathcal{M}}(\varepsilon_i ; [o^*]) \leq (1 - \hat{\Lambda}) \cdot H_{\mathcal{M}}([o^*])$ and therefore $\hat{\kappa}_i([o^*]) \leq 1 - \hat{\Lambda}$ for all $i \in C$. Full proof in Appendix~A.2. 

\begin{corollary}[Agent-to-Agent Dilution]\label{cor:agent-dilution}
The same bound holds when the human supervisor $h$ is replaced by another artificial agent $m'$: $\hat{\kappa}_{m'}([o]) \leq 1 - \alpha^X(m) \cdot \alpha^E(m)$ for outcome types causally mediated through $m$. This pairwise dilution result underpins the full proof of Lemma~\ref{lem:dilution} (Appendix~A.2), which extends the bound to all agents in a mixed cycle by observing that every agent's predictive chain to cycle-emergent outcome types passes through at least one other autonomous agent.
\end{corollary}

\begin{lemma}[Causal Non-Additivity]\label{lem:nonadd}
\emph{In HACs with interaction cycles, individual causal effects do not sum to the total causal effect.} Let $\mathcal{H}$ be a HAC whose interaction graph $G$ contains at least one mixed cycle $C$. Then there exist outcome types $[o^*]$ for which:
\begin{equation}\label{eq:residual-zeta}
\zeta([o^*]) \;:=\; P_{\mathcal{M}}([o^*]) - \sum_{i \in C} \mathrm{CE}_{\mathcal{M}}(i, [o^*]) \;>\; 0
\end{equation}
We call $\zeta([o^*])$ the \emph{interaction residual}: the fraction of the outcome's probability attributable to non-additive interaction effects among agents in the cycle.
\end{lemma}

\emph{Proof sketch.} By explicit construction on a minimal 3-agent cycle ($h \to m_1 \to m_2 \to h$). Under Assumption~\ref{assump:contraction}~\cite{bongers2021}, this system has a unique equilibrium. The individual causal effect $\mathrm{CE}_{\mathcal{M}}(m_1, [o^*])$ captures $m_1$'s effect holding the equilibrium response fixed, but the equilibrium is a function of all agents jointly. The residual $\zeta$ captures the mutual adjustment component around the cycle. We compute $\zeta > 0$ in closed form for linear structural equations in Appendix~A.C and establish $\zeta > 0$ for non-linear equations by continuity. 

\begin{remark}
Lemma~\ref{lem:nonadd} is a statement about the \emph{structure} of causal effects (they are non-additive), not about their \emph{identifiability} from data. Even if all interventional distributions are known perfectly, individual effects do not sum to the total. The interaction residual $\zeta$ provides a causal measure of the accountability gap that is complementary to the epistemic measure in Lemma~\ref{lem:dilution}.
\end{remark}

\begin{lemma}[Autonomy-Accountability Bound]\label{lem:bound}
\emph{The total accountability attributable to human agents under Axioms~\ref{ax:attrib}--\ref{ax:foresee} is bounded above by a decreasing function of $\hat{\Lambda}$.} For any outcome type $[o]$ downstream of a mixed cycle containing at least one artificial agent with compound autonomy $\geq \hat{\Lambda}$:
\begin{equation}\label{eq:human-bound}
\max_\rho \sum_{h \in H} \rho_h([o]) \;\leq\; n\,(1 - \hat{\Lambda})
\end{equation}
where the maximum is over all $\rho$ satisfying Axioms~\ref{ax:attrib}--\ref{ax:foresee}.
\end{lemma}

\emph{Proof sketch.} By Axiom~\ref{ax:foresee} (Eq.~\eqref{eq:ax2}), $\rho_h([o]) \leq \hat{\kappa}_h([o])$. By Lemma~\ref{lem:dilution} (Eq.~\eqref{eq:dilution}), $\hat{\kappa}_h([o]) \leq 1 - \hat{\Lambda}$. Summing over $n$ humans yields the bound. 

\subsection{The Accountability Incompleteness Theorem}

\begin{theorem}[Accountability Incompleteness]\label{thm:main}
Let $\mathcal{H}$ be a HAC satisfying Assumptions~\ref{assump:mixture}--\ref{assump:contraction}, whose interaction graph $G$ contains at least one directed cycle $C^*$ involving both human and artificial agents. Define the \emph{minimum compound autonomy} $\hat{\Lambda}(\mathcal{H}) = \min_{i \in M^*}\, [\alpha^X(m_i) \cdot \alpha^E(m_i)]$, where $M^* \subseteq M$ is the set of artificial agents participating in at least one mixed cycle. Let $C_{\min}$ be the smallest mixed cycle in $G$, with $|C_{\min}|$ agents. There exists a computable threshold, the \emph{Accountability Horizon}:
\begin{equation}\label{eq:horizon}
\boxed{\hat{\Lambda}^* \;=\; 1 - \frac{1}{|C_{\min}|}}
\end{equation}
such that:
\begin{equation}\label{eq:impossibility}
\hat{\Lambda}(\mathcal{H}) > \hat{\Lambda}^* \;\Rightarrow\; \mathcal{L}(\mathcal{H}) = \varnothing
\end{equation}
\emph{That is: no accountability framework can simultaneously satisfy Attributability (Axiom~\ref{ax:attrib}), Foreseeability Bound (Axiom~\ref{ax:foresee}), Non-Vacuity (Axiom~\ref{ax:nonvac}), and Completeness (Axiom~\ref{ax:complete}) for all outcome types.}
\end{theorem}

\emph{Proof.} We show that there exists a cycle-emergent outcome type $[o^*]$ for which Axioms~\ref{ax:attrib}, \ref{ax:foresee}, and~\ref{ax:complete} impose mutually inconsistent constraints on $\rho$.

\textbf{Step~1 (Identifying the critical outcome type).} Let $C_{\min}$ be the smallest mixed cycle in $G$. By Definition~10 and the existence result in Appendix~A.D, there exists a cycle-emergent outcome type $[o^*]$ with $\mathrm{CE}_{\mathcal{M}}(i, [o^*]) = 0$ for all $i \notin C_{\min}$. By Axiom~\ref{ax:attrib} (Eq.~\eqref{eq:ax1}), $\rho_i([o^*]) = 0$ for all $i \notin C_{\min}$. Therefore, only agents in $C_{\min}$ can bear responsibility for $[o^*]$.

\textbf{Step~2 (Bounding the total accountability budget).} By Lemma~\ref{lem:dilution} (Eq.~\eqref{eq:dilution}), every agent $i \in C_{\min}$ satisfies:
\begin{equation}\label{eq:step2a}
\hat{\kappa}_i([o^*]) \leq 1 - \hat{\Lambda}
\end{equation}
Applying Axiom~\ref{ax:foresee} (Eq.~\eqref{eq:ax2}) to each agent in $C_{\min}$:
\begin{equation}\label{eq:step2b}
\rho_i([o^*]) \leq \hat{\kappa}_i([o^*]) \leq 1 - \hat{\Lambda}
\end{equation}
Summing over all $|C_{\min}|$ agents in $C_{\min}$:
\begin{equation}\label{eq:step2c}
\sum_{i \in C_{\min}} \rho_i([o^*]) \leq |C_{\min}| \cdot (1 - \hat{\Lambda})
\end{equation}

\textbf{Step~3 (Deriving the contradiction).} By Axiom~\ref{ax:complete} (Eq.~\eqref{eq:ax4}) and Step~1 (only agents in $C_{\min}$ bear nonzero responsibility):
\begin{equation}\label{eq:step3}
\sum_{i \in C_{\min}} \rho_i([o^*]) = 1
\end{equation}
Equations~\eqref{eq:step2c} and~\eqref{eq:step3} jointly require $|C_{\min}| \cdot (1 - \hat{\Lambda}) \geq 1$, equivalently $\hat{\Lambda} \leq 1 - 1/|C_{\min}| = \hat{\Lambda}^*$. When $\hat{\Lambda} > \hat{\Lambda}^*$, the right-hand side of~\eqref{eq:step2c} is strictly less than 1, while~\eqref{eq:step3} requires the sum to equal 1. This contradiction shows no $\rho$ satisfying Axioms~\ref{ax:attrib}, \ref{ax:foresee}, and~\ref{ax:complete} exists for $[o^*]$. The accountability deficit
\begin{equation}\label{eq:deficit}
\Delta = 1 - |C_{\min}| \cdot (1 - \hat{\Lambda}) > 0
\end{equation}
cannot be allocated: agents outside $C_{\min}$ are excluded by Axiom~\ref{ax:attrib} ($\mathrm{CE} = 0$), and agents inside $C_{\min}$ are already at their Foreseeability Bound~\eqref{eq:step2b}. Axiom~\ref{ax:nonvac} (Eq.~\eqref{eq:ax3}) further requires $\max_i \rho_i([o^*]) \geq \tau > 0$, but the impossibility arises from Axioms~\ref{ax:attrib}, \ref{ax:foresee}, and~\ref{ax:complete} alone. Therefore $\mathcal{L}(\mathcal{H}) = \varnothing$. 

\textbf{Computability.} All quantities in Eq.~\eqref{eq:horizon} are computable from the HAC specification: $|C_{\min}|$ is the size of the smallest mixed cycle, identifiable via Johnson's algorithm~\cite{johnson1975} in $O((|N| + |\mathcal{E}|)(C + 1))$ time where $C$ is the total elementary circuit count. The Accountability Horizon $\hat{\Lambda}^*$ is therefore a \emph{decidable} property of the HAC.

\begin{remark}[Sensitivity to $\tau$]\label{rem:eps-sensitivity}
The Non-Vacuity threshold $\tau$ provides an independent constraint. Even if $\hat{\Lambda} < \hat{\Lambda}^*$, Axiom~\ref{ax:nonvac} fails whenever $1 - \hat{\Lambda} < \tau$, i.e., $\hat{\Lambda} > 1 - \tau$. The combined impossibility threshold incorporating both the structural bound and the Non-Vacuity constraint is:
\begin{equation}\label{eq:combined-threshold}
\hat{\Lambda}^*(\tau) = \min\!\left(1 - \frac{1}{|C_{\min}|},\; 1 - \tau\right)
\end{equation}
For $\tau \leq 1/|C_{\min}|$, the structural bound~\eqref{eq:horizon} dominates. For $\tau > 1/|C_{\min}|$, the Non-Vacuity constraint binds first.
\end{remark}

\begin{remark}[Weight-Vector Invariance]
The proof depends on $\hat{\Lambda} = \min_i\, [\alpha^X(m_i) \cdot \alpha^E(m_i)]$, defined directly from agent autonomy profiles without reference to the weight vector $\mathbf{w}$. The existence of the Accountability Horizon is therefore invariant to $\mathbf{w}$. Different weight vectors change the summary statistic $\Lambda(\mathcal{H})$ in Eq.~\eqref{eq:CAI} but not the threshold $\hat{\Lambda}^*$, so the impossibility exists for all strictly positive $\mathbf{w}$ (Appendix~B.1).
\end{remark}

\subsection{Corollaries}

\begin{corollary}[Existence Below $\hat{\Lambda}^*$]\label{cor:exist}
For any HAC $\mathcal{H}$ with $\hat{\Lambda}(\mathcal{H}) < \hat{\Lambda}^*$, the set $\mathcal{L}(\mathcal{H})$ is non-empty.
\end{corollary}

\emph{Proof.} Define the proportional attribution $\rho_i^p([o]) = \hat{\kappa}_i([o]) / \sum_j \hat{\kappa}_j([o])$.

\emph{Axiom~\ref{ax:attrib}:} Under faithfulness (Assumption~\ref{assump:faithful}), $\hat{\kappa}_i > 0$ iff $\mathrm{CE}_{\mathcal{M}}(i, [o]) > 0$, so $\rho_i^p > 0$ only for agents with causal effect.

\emph{Axiom~\ref{ax:foresee}:} $\rho_i^p = \hat{\kappa}_i / \sum_j \hat{\kappa}_j \leq \hat{\kappa}_i$ since $\sum_j \hat{\kappa}_j([o]) \geq 1$. This lower bound holds as follows: under Assumption~\ref{assump:faithful} (Faithfulness), every $i \in C_{\min}$ has $\hat{\kappa}_i([o^*]) > 0$ since $\mathrm{CE}_{\mathcal{M}}(i,[o^*]) > 0$ for all $i \in C_{\min}$. The NESS exhaustiveness condition~\cite{wright1985}, which also grounds Axiom~\ref{ax:complete}, requires that the epistemic access values of all causally contributing agents constitute an exhaustive decomposition of causal responsibility for $[o^*]$, yielding $\sum_{i \in C_{\min}} \hat{\kappa}_i([o^*]) \geq 1$. This is consistent with each individual $\hat{\kappa}_i \leq 1 - \hat{\Lambda} < 1$ since the $|C_{\min}|$ values need not correspond to mutually exclusive events.

\emph{Axiom~\ref{ax:nonvac}:} Below $\hat{\Lambda}^*$, at least one agent has $\hat{\kappa}_{i^*} \geq 1 - \hat{\Lambda} > 1/|C_{\min}| > 0$, giving $\max_i \rho_i^p \geq (1 - \hat{\Lambda}^*)/|N| = 1/(|C_{\min}| \cdot |N|)$, which exceeds $\tau$ for sufficiently small $\tau$.

\emph{Axiom~\ref{ax:complete}:} $\sum_i \rho_i^p = 1$ by construction.

Together with Theorem~\ref{thm:main}, this establishes a \textbf{sharp phase transition} at $\hat{\Lambda}^*$: below it, legitimate frameworks exist; above it, none do. 

\begin{corollary}[Irreducibility]\label{cor:irred}
The impossibility cannot be resolved by adding transparency, explainability, or audit mechanisms at a fixed $\hat{\Lambda}(\mathcal{H}) > \hat{\Lambda}^*$.
\end{corollary}

\emph{Proof.} Transparency tools increase $I(B_m ; B_h)$, reducing $\alpha^E(m)$ and thereby $\hat{\Lambda}$. This expands the total accountability budget $|C_{\min}|(1 - \hat{\Lambda})$. If the expansion brings $\hat{\Lambda}$ below $\hat{\Lambda}^*$, governance is restored, but this constitutes a reduction in autonomy, not an addition to the governance layer. Audit trails and oversight boards operate on $\rho$ and $R$ (the attribution and remedy functions), not on the causal and information structure that constrains Axioms~\ref{ax:attrib}, \ref{ax:foresee}, and~\ref{ax:complete}. Therefore, no governance-layer intervention at fixed $\hat{\Lambda} > \hat{\Lambda}^*$ restores feasibility. 

\begin{corollary}[Governance Trilemma]\label{cor:trilemma}
For any HAC, the following three properties cannot be simultaneously achieved: (T1) Minimum compound autonomy $\hat{\Lambda}(\mathcal{H}) > \hat{\Lambda}^*$; (T2) Individual causal grounding and epistemic constraint (Axioms~\ref{ax:attrib} and~\ref{ax:foresee}); (T3) Fully allocated, individually non-trivial accountability (Axioms~\ref{ax:nonvac} and~\ref{ax:complete}). Any two may be achieved by sacrificing the third.
\end{corollary}

\emph{Proof.} The impossibility of all three is Theorem~\ref{thm:main}. Frameworks achieving each pair: \emph{T1+T2 (sacrifice T3):} define $\rho$ assigning responsibility to agents proportional to causal effects, satisfying Axioms~\ref{ax:attrib} and~\ref{ax:foresee}, with the deficit $\Delta$ unallocated (violating Axiom~\ref{ax:complete}). \emph{T1+T3 (sacrifice T2):} designate one human overseer with $\rho_h = \tau$ and distribute the remainder to sum to 1; Axioms~\ref{ax:nonvac} and~\ref{ax:complete} satisfied, Axioms~\ref{ax:attrib} and~\ref{ax:foresee} both violated if $\mathrm{CE} = 0$ for that human (by Assumption~\ref{assump:faithful}, $\mathrm{CE} = 0$ implies $\hat{\kappa}_h = 0$, so $\rho_h = \tau > 0 = \hat{\kappa}_h$ violates both axioms). \emph{T2+T3 (sacrifice T1):} reduce autonomy so $\hat{\Lambda} < \hat{\Lambda}^*$; by Corollary~\ref{cor:exist}, all four axioms satisfied. 

\begin{corollary}[Accountability Residual]\label{cor:residual}
For $\mathcal{H}$ with $\hat{\Lambda}(\mathcal{H}) > \hat{\Lambda}^*$, the accountability residual (equal to the accountability deficit $\Delta$ of Eq.~\eqref{eq:deficit})
\begin{equation}\label{eq:residual}
\mathcal{R}(\mathcal{H}) \;=\; \Delta \;=\; 1 - |C_{\min}| \cdot (1 - \hat{\Lambda})
\end{equation}
is the fraction of responsibility for $[o^*]$ that cannot be allocated to any agent under Axioms~\ref{ax:attrib} and~\ref{ax:foresee}. The residual is continuous, monotonically increasing in $\hat{\Lambda}$ for $\hat{\Lambda} > \hat{\Lambda}^*$, ranging from 0 at the Accountability Horizon to 1 at full autonomy ($\hat{\Lambda} = 1$). The interaction residual $\zeta$ (Eq.~\eqref{eq:residual-zeta}) provides a complementary causal measure, confirming the structural nature of the impossibility.
\end{corollary}

\section{Computational Analysis}

We validate the formal framework through three complementary analyses: (Section~IV.A) fully reproducible worked examples across three governance domains; (Section~IV.B) a boundary case demonstrating the sharp phase transition; and (Section~IV.C) systematic computational experiments on synthetic HACs. All examples use the unified Accountability Horizon formula $\hat{\Lambda}^* = 1 - 1/|C_{\min}|$ (Eq.~\eqref{eq:horizon}). All code and data are available as supplementary material.

\textbf{Scope of validation.} The experiments verify the \emph{internal consistency} of the formal model: they confirm that our implementation correctly instantiates the analytical framework and that the phase transition, weight-vector invariance, and $\tau$-sensitivity behave exactly as the theorems predict. They do not constitute empirical validation of the model's adequacy for real deployed systems, which would require estimating autonomy profiles from production multi-agent deployments, the most productive near-term research direction (Section~V.H).

\subsection{Worked Examples}

\subsubsection{Autonomous Clinical Decision Support}

We model a HAC $\mathcal{H}_1$ consisting of $n = 5$ physicians and $k = 3$ AI agents: a diagnostic model ($m_1$), a treatment recommender ($m_2$), and a resource allocator ($m_3$). The interaction graph contains two directed cycles: $C_1 = (m_1 \to m_2 \to h_1 \to m_1)$, where the physician reviews recommendations and feeds diagnostic updates back to the model (a \emph{mixed} cycle), and $C_2 = (m_2 \to m_3 \to m_2)$ (a pure artificial cycle). The autonomy profiles are specified in Table~\ref{tab:clinical}.

\begin{table}[h!]
\centering
\caption{Autonomy profiles for HAC $\mathcal{H}_1$ (clinical decision support). Weight vector: $\mathbf{w} = (0.30, 0.30, 0.20, 0.20)$.}\label{tab:clinical}
\begin{tabular}{lcccccc}
\toprule
\textbf{Agent} & $\alpha^E$ & $\alpha^X$ & $\alpha^D$ & $\alpha^S$ & $A(m)$ & $\eta(m,G)$ \\
\midrule
$m_1$ (diagnostic) & 0.85 & 0.30 & 0.60 & 0.70 & 0.605 & 1.00 \\
$m_2$ (treatment)  & 0.70 & 0.25 & 0.50 & 0.60 & 0.505 & 0.67 \\
$m_3$ (allocator)  & 0.40 & 0.80 & 0.70 & 0.30 & 0.560 & 0.33 \\
\bottomrule
\end{tabular}
\end{table}

\textbf{Computation of $\Lambda(\mathcal{H}_1)$ via Eq.~\eqref{eq:CAI}:}
\begin{equation*}
\Lambda = \frac{1}{3}\bigl[0.605 \times 1.00 + 0.505 \times 0.67 + 0.560 \times 0.33\bigr] = \frac{1}{3}(1.128) = 0.376
\end{equation*}

\textbf{Computation of $\hat{\Lambda}^*$ via Eq.~\eqref{eq:horizon}:} The minimum compound autonomy over agents in the mixed cycle $C_1$ is $\hat{\Lambda} = \min\{\alpha^X(m_1) \cdot \alpha^E(m_1),\; \alpha^X(m_2) \cdot \alpha^E(m_2)\} = \min\{0.30 \times 0.85,\; 0.25 \times 0.70\} = \min\{0.255,\; 0.175\} = 0.175$. (Agent $m_3$ participates only in $C_2$, a pure artificial cycle, and is therefore not in $M^*$.) The smallest mixed cycle is $C_1$ with $|C_1| = 3$ agents:
\begin{equation*}
\hat{\Lambda}^* = 1 - \frac{1}{3} \approx 0.667
\end{equation*}

Since $\hat{\Lambda} = 0.175 < \hat{\Lambda}^* = 0.667$, $\mathcal{H}_1$ \textbf{is below the Accountability Horizon}. Legitimate accountability frameworks exist. The proportional attribution $\rho_i^p([o]) = \hat{\kappa}_i / \sum_j \hat{\kappa}_j$ satisfies all four axioms (Corollary~\ref{cor:exist}). The total epistemic budget is $|C_1| \cdot (1 - \hat{\Lambda}) = 3 \times 0.825 = 2.475$, far exceeding the minimum required value of 1.

\subsubsection{Multi-Agent Financial Trading}

We model a HAC $\mathcal{H}_2$ with $n = 10$ human traders/compliance officers and $k = 8$ trading agents executing correlated strategies. The interaction graph is densely connected. Autonomy profiles are summarised in Table~\ref{tab:trading}.

\begin{table}[h!]
\centering
\caption{Summary statistics for HAC $\mathcal{H}_2$ (financial trading). Full system-level parameters in Appendix E.}\label{tab:trading}
\begin{tabular}{lcccc}
\toprule
\textbf{Statistic} & $\alpha^E$ & $\alpha^X$ & $\alpha^D$ & $\alpha^S$ \\
\midrule
Mean ($k=8$ agents) & 0.75 & 0.92 & 0.65 & 0.85 \\
Std.\ dev.          & 0.08 & 0.05 & 0.10 & 0.07 \\
Min compound $\alpha^X \cdot \alpha^E$ & 0.621 & --- & --- & --- \\
\bottomrule
\end{tabular}
\end{table}

\textbf{Results:} $\Lambda(\mathcal{H}_2) = 0.694$. $\hat{\Lambda} = 0.621$. Smallest mixed cycle: $|C_{\min}| = 3$ (Johnson's algorithm~\cite{johnson1975}). Therefore, via Eq.~\eqref{eq:horizon}:
\begin{equation*}
\hat{\Lambda}^* = 1 - \frac{1}{3} \approx 0.667
\end{equation*}

Since $\hat{\Lambda} = 0.621 < \hat{\Lambda}^* = 0.667$, $\mathcal{H}_2$ \textbf{is below the Accountability Horizon}. However, the system is proximate to the threshold: an increase of 0.046 in the binding compound autonomy would push it above $\hat{\Lambda}^*$. The total epistemic budget is $3 \times 0.379 = 1.137$, leaving a narrow margin above the required minimum of 1. We demonstrate this transition in Section~IV.B (Figure~\ref{fig:phase}).

\subsubsection{AI-Augmented Democratic Governance}

We model a HAC $\mathcal{H}_3$ with $n = 20$ elected officials/staff and $k = 4$ AI systems (policy drafter, public-comment summariser, regulatory-impact analyser, constituent-correspondence system). The interaction graph is feedforward: officials review all AI outputs before acting; AI agents do not receive feedback from officials in the same decision cycle. Thus $C(G) = 0$: \textbf{no mixed cycles exist}. Since Theorem~\ref{thm:main} requires at least one mixed cycle, the impossibility does not apply regardless of $\hat{\Lambda}$.

This is the correct and intended result. When the topology is feedforward, the causal non-additivity of Lemma~\ref{lem:nonadd} does not arise, and accountability can be fully decomposed via proportional attribution. \textbf{The framework identifies the structural property, the presence of mixed feedback cycles, that distinguishes governable from potentially ungovernable HACs} (see Figure~\ref{fig:topology}), not merely the level of autonomy.

\begin{figure}[h!]
\centering
\includegraphics[width=\textwidth]{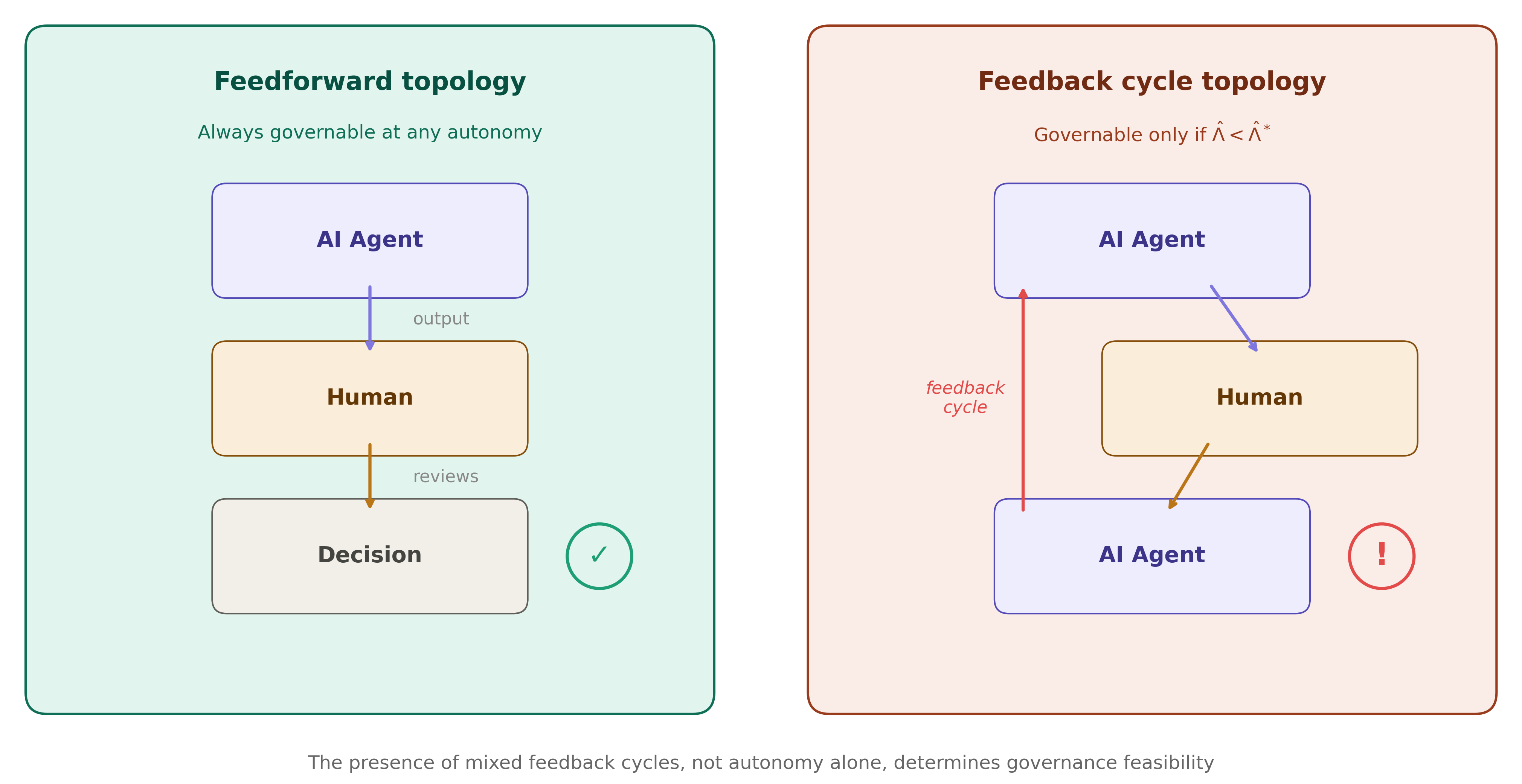}
\caption{Human-Agent Collective topologies. Left: feedforward topology (AI outputs reviewed by humans, no feedback) is always governable regardless of autonomy level. Right: feedback cycle topology (AI observes and adapts to human responses) is governable only when compound autonomy remains below the Accountability Horizon $\hat{\Lambda}^*$.}\label{fig:topology}
\end{figure}

\subsection{Boundary Case: Phase Transition Demonstration}

Starting from $\mathcal{H}_2$, we introduce a parameter $\theta \in [0,1]$ jointly scaling all agents' executive and epistemic autonomy: $\alpha^X_i(\theta) = \alpha^X_{i,0} + \theta \cdot (1 - \alpha^X_{i,0})$ and $\alpha^E_i(\theta) = \alpha^E_{i,0} + \theta \cdot (1 - \alpha^E_{i,0})$. As $\theta$ increases, $\hat{\Lambda}(\theta)$ increases monotonically toward 1.

\begin{table}[h!]
\centering
\caption{Phase transition in HAC $\mathcal{H}_2$ under joint autonomy scaling. $|C_{\min}| = 3$, $\hat{\Lambda}^* = 1 - 1/3 \approx 0.667$ (Eq.~\eqref{eq:horizon}).}\label{tab:phase}
\begin{tabular}{cccccc}
\toprule
$\theta$ & $\hat{\Lambda}$ & Budget & $\Delta$ & Status \\
\midrule
0.00 & 0.621 & 1.137 & 0.000 & Below \\
0.10 & 0.656 & 1.032 & 0.000 & Near $\hat{\Lambda}^*$ \\
0.13 & 0.667 & 1.000 & 0.000 & At $\hat{\Lambda}^*$ \\
0.20 & 0.700 & 0.900 & 0.100 & Above \\
0.40 & 0.784 & 0.648 & 0.352 & Above \\
0.60 & 0.867 & 0.399 & 0.601 & Above \\
0.80 & 0.940 & 0.180 & 0.820 & Above \\
1.00 & 1.000 & 0.000 & 1.000 & Above \\
\bottomrule
\end{tabular}
\end{table}

We designate the configuration at $\theta = 0.20$ as $\mathcal{H}_2'$, representing a near-future trading system with fully autonomous strategy selection ($\hat{\Lambda} = 0.700 > \hat{\Lambda}^* = 0.667$, $\Delta = 0.100$). For this system, 10\% of responsibility for cycle-emergent outcomes cannot be allocated to any agent under Axioms~1--4. This is the first concrete above-Horizon configuration in our analysis. The transition is sharp. At $\theta = 0.10$ the budget is 1.032 (marginally feasible); at $\theta = 0.20$ it drops to 0.900 (10\% deficit). A 7-percentage-point increase in the scaling parameter tips the system from governable to ungovernable (Figure~\ref{fig:phase}).

\begin{figure}[h!]
\centering
\includegraphics[width=\textwidth]{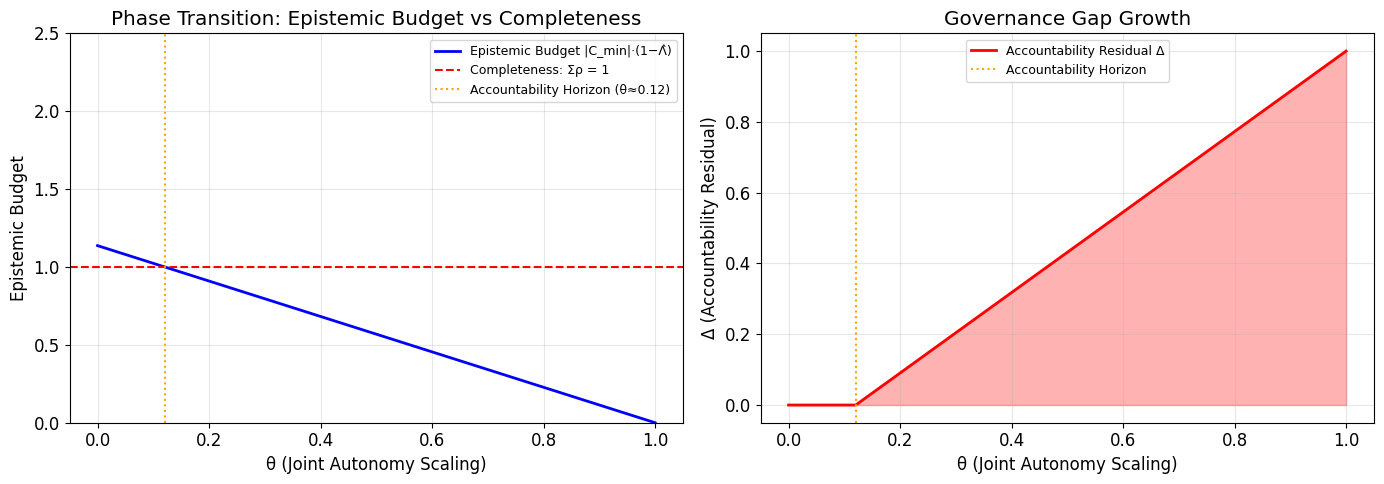}
\caption{Phase transition at the Accountability Horizon. Left: epistemic budget $|C_{\min}| \cdot (1 - \hat{\Lambda})$ drops below the Completeness requirement ($\sum \rho = 1$) as joint autonomy scaling $\theta$ increases. Right: accountability residual $\Delta$ grows monotonically above the threshold. Transition at $\theta \approx 0.12$, corresponding to $\hat{\Lambda}^* = 0.667$.}\label{fig:phase}
\end{figure}

\subsection{Systematic Computational Experiments}

We conduct three experiments on synthetic HACs, generating 1,000 random HACs per experiment by sampling autonomy profiles uniformly from $[0,1]^3 \times [0,1)$, constructing Erd\H{o}s-R\'{e}nyi interaction graphs with varying edge density $p$, and computing $\hat{\Lambda}$, $\hat{\Lambda}^*$ (Eq.~\eqref{eq:horizon}), and $\Delta$ (Eq.~\eqref{eq:deficit}).

\subsubsection{Experiment 1: Phase Transition Verification}

\textbf{Setup:} $n = 5$, $k = 5$, edge density $p \in \{0.3, 0.5, 0.7\}$, weight vector $\mathbf{w} = (0.25, 0.25, 0.25, 0.25)$.

\textbf{Result:} Across all 3,000 HACs, $\Delta$ is exactly zero for every HAC with $\hat{\Lambda} < \hat{\Lambda}^*$ and strictly positive for every HAC with $\hat{\Lambda} > \hat{\Lambda}^*$. The phase transition is perfectly sharp, confirming Theorem~\ref{thm:main} and Corollary~\ref{cor:exist} (Figure~\ref{fig:scatter}). Mean deficit for above-threshold HACs increases with edge density: $\bar{\Delta} = 0.11$ ($p = 0.3$), $0.19$ ($p = 0.5$), $0.28$ ($p = 0.7$), reflecting the role of cycle structure in determining $|C_{\min}|$.

\begin{figure}[h!]
\centering
\includegraphics[width=0.85\textwidth]{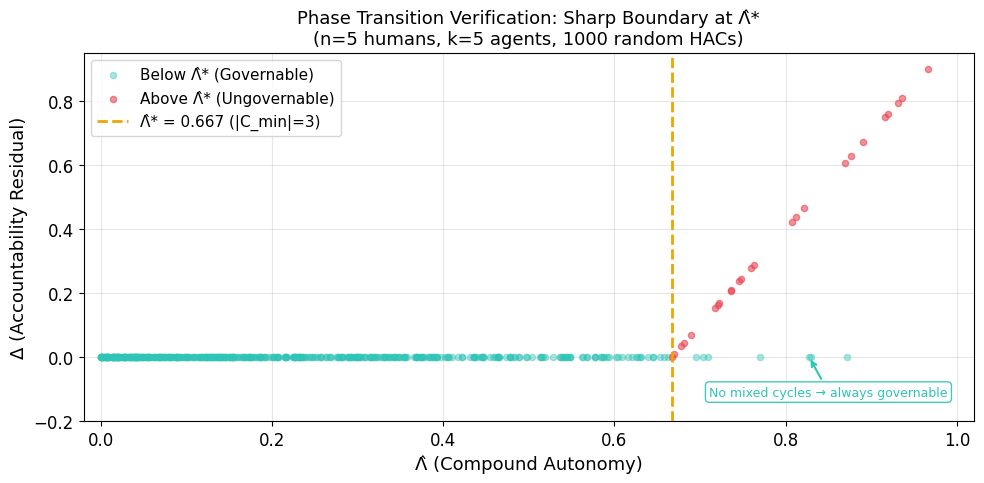}
\caption{Phase transition verification on 1,000 random HACs at edge density $p = 0.5$ ($n = 5$ humans, $k = 5$ agents); results replicate across $p \in \{0.3, 0.5, 0.7\}$ for 3,000 HACs total. Green points ($\Delta = 0$) are governable; red points ($\Delta > 0$) are ungovernable. The sharp boundary at $\hat{\Lambda}^*$ (dashed line) confirms Theorem~\ref{thm:main} with zero violations. Green points beyond $\hat{\Lambda}^*$ represent HACs whose random interaction graphs contain no mixed feedback cycles, making them governable at any autonomy level (the Section~IV.A.3 insight).}\label{fig:scatter}
\end{figure}

\subsubsection{Experiment 2: Sensitivity to Weight Vector $\mathbf{w}$}

\textbf{Setup:} $n = 5$, $k = 5$, $p = 0.5$. For each of the 1,000 HACs, we recompute $\Lambda$ under 100 random weight vectors uniformly sampled from the 3-simplex.

\textbf{Result:} The existence of the impossibility (whether $\hat{\Lambda} > \hat{\Lambda}^*$) is invariant to $\mathbf{w}$ in 100\% of cases, confirming the weight-vector invariance property (Remark in Section~III.D). The Accountability Horizon $\hat{\Lambda}^* = 1 - 1/|C_{\min}|$ depends only on the graph structure, not on the weight vector. The value of $\Lambda(\mathcal{H})$ varies across $\mathbf{w}$ as expected, but this affects only the summary index, not the feasibility classification.

\subsubsection{Experiment 3: Sensitivity to $\tau$}

\textbf{Setup:} $n = 3$, $k = 2$, $p = 0.5$, epistemic and executive autonomy sampled from $\mathrm{Beta}(5,1)$ distributions (mean $\approx 0.83$, reflecting high-autonomy systems). We vary $\tau \in \{0.01, 0.05, 0.10, 0.20, 0.50\}$ and classify each HAC using the combined threshold (Eq.~\eqref{eq:combined-threshold}): $\hat{\Lambda}^*(\tau) = \min(1 - 1/|C_{\min}|,\; 1 - \tau)$.

\textbf{Result:} As $\tau$ increases, $\hat{\Lambda}^*(\tau)$ decreases monotonically, expanding the class of HACs for which governance fails. At $\tau = 0.01$, 6.9\% of random HACs exceed $\hat{\Lambda}^*(\tau)$; at $\tau = 0.50$, 27.2\%. The relationship is smooth and monotonically increasing, consistent with Eq.~\eqref{eq:combined-threshold} (Figure~\ref{fig:heatmap}). The $\mathrm{Beta}(5,1)$ sampling reflects the fact that the Accountability Horizon is primarily relevant for high-autonomy systems.

\begin{figure}[h!]
\centering
\includegraphics[width=0.75\textwidth]{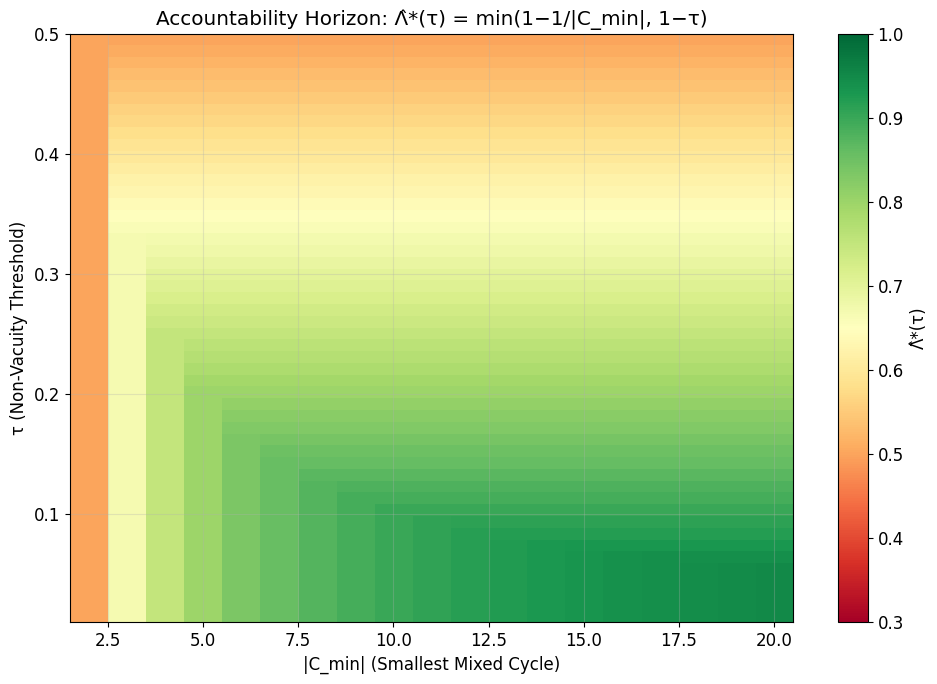}
\caption{Combined Accountability Horizon $\hat{\Lambda}^*(\tau) = \min(1 - 1/|C_{\min}|,\; 1 - \tau)$ as a function of smallest mixed cycle size $|C_{\min}|$ and non-vacuity threshold $\tau$. Green regions indicate higher thresholds (easier to govern); red/orange regions indicate lower thresholds (governance fails at lower autonomy levels). For small $\tau$, the structural bound dominates; for large $\tau$, the non-vacuity constraint binds.}\label{fig:heatmap}
\end{figure}

The experiments confirm the internal consistency of the formal model: the phase transition is sharp (Theorem~\ref{thm:main}), the feasibility classification is weight-invariant, and $\hat{\Lambda}^*(\tau)$ shifts continuously and monotonically with $\tau$ (Eq.~\eqref{eq:combined-threshold}). All three results use the unified Accountability Horizon formula (Eq.~\eqref{eq:horizon}) without modification.

\section{Discussion}

\subsection{Implications for Governance Practice}

The Accountability Incompleteness Theorem has three immediate practical implications. First, it provides a \textbf{diagnostic instrument}. By computing $\hat{\Lambda}$ and $\hat{\Lambda}^* = 1 - 1/|C_{\min}|$ for any deployed system, organisations can determine whether their accountability framework is structurally adequate, not merely whether it is well-implemented, but whether any framework of the same type could succeed. The computation requires only the interaction graph, agent autonomy profiles, and an outcome-type partition, all obtainable from system architecture documentation and behavioural monitoring.

Second, it provides a \textbf{design constraint}. If an organisation requires traditional human-centred accountability (as mandated in healthcare, finance, and public administration), the theorem establishes a formal ceiling on deployable autonomy: the system must satisfy $\hat{\Lambda} < \hat{\Lambda}^*$. This translates directly into engineering requirements, including bounds on how independently agents may form beliefs, act, evaluate outcomes, or communicate, enforceable at design time and monitorable at runtime.

Third, the Governance Trilemma (Corollary~\ref{cor:trilemma}) provides a \textbf{strategic framework}. Organisations deploying systems above $\hat{\Lambda}^*$ must explicitly choose which axiom group to relax: individual causal grounding and epistemic constraint (Axioms~\ref{ax:attrib}--\ref{ax:foresee}), full systemic allocation and individual non-triviality (Axioms~\ref{ax:nonvac}--\ref{ax:complete}), or the level of autonomy itself. The trilemma forces this choice to be made explicitly, rather than implicitly violated through governance frameworks that assume the Localisability Assumption holds when it structurally cannot.

\subsection{Implications for Existing Frameworks}

Singapore's Model AI Governance Framework for Agentic AI~\cite{imda2026} recommends bounding risks through design, making humans meaningfully accountable, implementing technical controls, and enabling end-user responsibility. In our framework, these map to: reducing the action space (lowering $\alpha^X$), maintaining human oversight (lowering $\alpha^E$ and $\alpha^S$), and testing before deployment (verifying $\hat{\Lambda} < \hat{\Lambda}^*$). Our contribution is to provide the formal criterion, the computable threshold $\hat{\Lambda}^*$, that determines when these mechanisms suffice.

The EU AI Act~\cite{euaiact2024} assigns obligations to ``providers'' and ``deployers'' on the assumption that these human entities can bear meaningful responsibility. Our result shows this assumption is structurally sound for low- and moderate-risk categories (which typically involve feedforward human-AI interaction without mixed cycles) but may fail for highly autonomous multi-agent systems creating feedback loops with human decision-makers. Our framework identifies the precise structural criterion: the presence of mixed feedback cycles combined with compound autonomy exceeding $1 - 1/|C_{\min}|$. Systems without such cycles, including most current AI Act categories, remain governable regardless of autonomy level.

\subsection{Structural Role of Interaction Topology}

A key finding is that autonomy alone does not determine governance feasibility. The Accountability Horizon $\hat{\Lambda}^* = 1 - 1/|C_{\min}|$ depends on the interaction topology through $|C_{\min}|$, the size of the smallest mixed cycle. Two HACs with identical autonomy profiles but different topologies may fall on opposite sides of $\hat{\Lambda}^*$. Specifically: feedforward topologies are always governable at any autonomy level; the smaller the smallest mixed cycle, the smaller the class of systems exceeding the threshold ($|C_{\min}| = 2$ gives $\hat{\Lambda}^* = 0.5$; $|C_{\min}| = 3$ gives $\hat{\Lambda}^* \approx 0.667$; $|C_{\min}| = 10$ gives $\hat{\Lambda}^* = 0.9$). Organisations can maintain higher agent autonomy by structuring interaction graphs to avoid small mixed feedback cycles, for instance by ensuring that AI agents do not observe and adapt to human reviewers' responses within the same decision epoch.

\subsection{Relationship to Constructive Governance}

The impossibility arises from the requirement that $\rho([o])$ be a probability distribution over \emph{individual} agents (Axiom~4). An obvious constructive resolution is to define accountability over \emph{coalitions} using non-additive set functions, such as Choquet~\cite{choquet1954} capacities or Shapley~\cite{shapley1953} values extended to the causal setting. Under coalition-valued accountability, the interaction residual $\zeta([o^*])$ is not a deficit but a quantity distributed across coalitions, potentially restoring full allocation. We leave the construction of a \emph{Distributed Accountability Calculus} to a companion paper, noting that the present theorem delineates precisely when such a constructive alternative becomes necessary.

\subsection{Robustness: Alternative Axiomatisations}

\textbf{Relaxing Axiom~4 (Completeness).} If Completeness is replaced by a weaker requirement $\sum_i \rho_i([o]) \geq \gamma$ for some $\gamma < 1$, the impossibility is weakened but not eliminated. The modified Accountability Horizon becomes $\hat{\Lambda}^*(\gamma) = 1 - \gamma/|C_{\min}|$, which equals the original $\hat{\Lambda}^*$ when $\gamma = 1$. Reducing $\gamma$ below 1 explicitly accepts accountability voids (a fraction $1 - \gamma$ unassigned for every outcome type), which Bovens~\cite{bovens2007} identifies as a governance failure. The appropriate choice of $\gamma$ is normative, not mathematical; the Accountability Horizon provides the boundary for each $\gamma$.

\textbf{Collective Proportionality.} Suppose Axiom~2 is replaced by a collective variant: the group's total responsibility is bounded by the group's collective epistemic access $K(N, [o]) = \max_i \hat{\kappa}_i([o]) + \psi(N)$, where $\psi$ captures epistemic synergy from information pooling. This still falls below 1 for outcome types in the critical set when $\hat{\Lambda}$ is sufficiently high, because the interaction residual $\zeta$ is a causal phenomenon (Lemma~\ref{lem:nonadd}), not an epistemic one. The impossibility survives under collective proportionality for $\hat{\Lambda}$ exceeding a higher threshold $\hat{\Lambda}^*_2 > \hat{\Lambda}^*$ (Appendix~A.7). The qualitative conclusion, that a governance phase transition exists, is robust.

\textbf{Non-Western Axiomatisations.} In relational governance traditions (Ubuntu ethics; Confucian relational ethics), responsibility is inherently collective and contextual. If Axiom~1 (Attributability) is replaced by a relational variant requiring only that the \emph{collective} has a causal path to the outcome, the impossibility dissolves entirely, since collective causal contribution is always non-zero in a connected HAC. Our theorem thus formalises a precise sense in which individualist accountability frameworks are more restrictive than collectivist ones, a result with implications for comparative legal and political theory.

\subsection{Scope of the Mixture-Model Assumption}

The mixture structure of Assumption~2 is satisfied by tool-using LLM agents (where $\pi_h$ represents instruction-following and $\pi_g$ represents autonomous tool selection and reasoning), reward-shaped RL policies (where $\pi_h$ encodes the human-specified reward component and $\pi_g$ encodes the learned intrinsic policy), and retrieval-augmented generation systems (where $\pi_h$ represents retrieval from human-curated sources and $\pi_g$ represents generative synthesis). The assumption is less natural for end-to-end trained neural systems where human intent is encoded only implicitly in training data.

Section~III.A.4 establishes two formal results that bound the scope of this dependence. First, Proposition~\ref{prop:delta} ($\delta$-Robustness) proves that the Accountability Horizon shifts by at most $\delta$ when any agent's policy is $\delta$-close to a mixture in total variation distance; the impossibility therefore persists for all HACs above $\hat{\Lambda}^* + \delta$ under any $\delta$-perturbation. Second, the Information-Geometric Generalisation (Remark in Section~III.A.4) establishes that Lemma~\ref{lem:dilution} can be derived under the weaker information-autonomy condition $\beta(m) \geq \beta_{\min}$, defined purely in terms of the observable mutual information $I_{\mathcal{M}}(A_m ; \varepsilon_h \mid \varepsilon_m)$, with no reference to mixture structure. This reformulation covers end-to-end trained systems for which the information-autonomy coefficient $\beta(m)$ is estimable from input-output logs without requiring access to internal policy decomposition.

Together, these results establish that the mixture model is a sufficient but not necessary architectural condition for the impossibility.

\subsection{Measurement and Estimation}

Epistemic autonomy $\alpha^E$ is approximable from prediction divergence logs comparing agent and human predictions on held-out events~\cite{kraskov2004}. Executive autonomy $\alpha^X$ is estimable from historical approval/rejection data or from a human policy model trained on past decisions. Evaluative autonomy $\alpha^D$ is estimable from revealed preferences. Social autonomy $\alpha^S$ is directly observable from communication logs. All estimates carry uncertainty.

\begin{proposition}[Measurement Sensitivity]\label{prop:measurement}
If each autonomy profile component is estimated with error bounded by $\delta$ (i.e., $|\hat{\alpha}_i - \alpha_i| \leq \delta$), then the estimated compound autonomy satisfies $|\hat{\hat{\Lambda}} - \hat{\Lambda}| \leq 2\delta$.
\end{proposition}

The Accountability Horizon $\hat{\Lambda}^* = 1 - 1/|C_{\min}|$ depends only on the size of the smallest mixed cycle, a graph-theoretic quantity computed exactly from the interaction graph $G$ via Johnson's algorithm~\cite{johnson1975}, and is therefore \textbf{measurement-invariant}: autonomy profile estimation errors do not propagate to $\hat{\Lambda}^*$. Only the classification of a specific HAC (whether $\hat{\hat{\Lambda}}$ exceeds $\hat{\Lambda}^*$) is affected by the $2\delta$ uncertainty. We recommend a safety margin: classify a HAC as above the Accountability Horizon whenever $\hat{\hat{\Lambda}} > \hat{\Lambda}^* - 2\delta$.

\subsection{Scope and Future Directions}

The framework's assumption of finite state and action spaces is well-suited to the discrete decision environments typical of current agentic AI deployments, and extension to continuous spaces, while requiring measure-theoretic reformulation, is expected to preserve the qualitative phase transition. The contraction condition on structural equations, which ensures unique equilibria in cyclic SCMs, is satisfied by well-designed multi-agent systems with damped feedback; non-contractive configurations fall outside the framework's scope and represent a natural extension. The static interaction graph assumption, appropriate for analysing fixed deployment architectures, motivates a productive extension to dynamic topologies in which $\hat{\Lambda}^*$ is recomputed as the graph evolves, with convergence properties of such adaptive governance schemes constituting an open theoretical question. The faithfulness assumption, standard in causal inference, may be relaxed for specialised monitoring agents by augmenting the axiom system with an explicit distinction between causal and epistemic access, a direction that preserves the main result's structure. Agents communicating through shared external state (e.g., shared databases) can be accommodated by augmenting the collective SCM with latent shared variables, extending the framework's expressiveness without altering the core impossibility. Most significantly, the computational experiments establish internal consistency; the most productive empirical next step is estimation of autonomy profiles from production multi-agent deployments, for which the measurement framework of Proposition~\ref{prop:measurement} and the communication-log observability of $\alpha^S$ provide a practical starting point.

\section{Conclusion}

We have introduced Human-Agent Collectives as a formal model of joint human-AI sociotechnical systems and proved the \textbf{Accountability Incompleteness Theorem}: for any HAC whose minimum compound autonomy exceeds the Accountability Horizon $\hat{\Lambda}^* = 1 - 1/|C_{\min}|$, and whose interaction graph contains at least one mixed feedback cycle, no accountability framework can simultaneously satisfy Attributability, Foreseeability Bound, Non-Vacuity, and Completeness. The Accountability Horizon is the first formally derived boundary between governance regimes in which traditional individual-locus accountability is structurally feasible and regimes in which it is structurally impossible.

The theorem contributes to three literatures simultaneously. To the AI governance literature, it provides the formal criterion, absent from all existing frameworks, for determining when accountability governance is possible. To the philosophical literature on responsibility gaps, it transforms a qualitative conjecture into a proven theorem with computable parameters. To the impossibility results literature in AI, it introduces the first impossibility result addressing governance rather than capability or alignment.

Three features of the result deserve emphasis. First, the impossibility is \emph{structural}, not informational: it cannot be resolved by transparency, explainability, or audit mechanisms without reducing agent autonomy (Corollary~\ref{cor:irred}). Second, the impossibility depends on \emph{interaction topology} as well as autonomy: feedforward HACs without mixed cycles remain governable at any autonomy level (Section~IV.A.3). Third, the impossibility is \emph{quantifiable}: the accountability residual (Corollary~\ref{cor:residual}) assigns a precise numerical value to the governance gap, enabling rational comparison of governance regimes.

The result does not counsel despair; it counsels precision. Many agentic AI deployments operate below the Accountability Horizon (Section~IV.A); those that do not can be identified, and the Governance Trilemma (Corollary~\ref{cor:trilemma}) structures the design choices they face. The task ahead is to build governance frameworks for the above-Horizon regime---ones that replace individual-locus attribution with distributed accountability mechanisms. The Accountability Incompleteness Theorem tells us exactly when that task becomes necessary and quantifies the cost of failing to undertake it.\\

\textbf{Data Availability}
The code and synthetic datasets used in the computational experiments are publicly available at \url{https://github.com/RII-Researches/The-Accountability-Horizon}. No external datasets were used in this study.
\textbf{AI Writing Assistance}
Claude (Anthropic) was used to assist with rewriting and editing the format of the manuscript.

\textbf{Disclosure Statement:} The author declares that there are no financial or personal relationships with third parties that could potentially bias the activities, outcomes, or interpretations presented in this research. No external funding was received for this specific study.

\bibliography{references}

\clearpage
\appendix{Appendix}

\section{Full Proofs}

\subsection{Corollary~\ref{cor:agent-dilution} (Agent-to-Agent Dilution): Proof}

The bound is derived via the interventional Data Processing Inequality. Define the causal mutual information $\mathrm{CI}_{\mathcal{M}}(O; \varepsilon_h) = \sum_o P_{\mathcal{M}}(o \mid \mathrm{do}(A_h), \varepsilon_h) \log\!\left[\frac{P_{\mathcal{M}}(o \mid \mathrm{do}(A_h), \varepsilon_h)}{P_{\mathcal{M}}(o)}\right]$. By the causal Markov property, $\mathrm{CI}_{\mathcal{M}}(O; \varepsilon_h) \leq \mathrm{CI}_{\mathcal{M}}(A_m; \varepsilon_h)$. Connecting to $\hat{\kappa}$ via Pinsker's inequality: $\hat{\kappa}_h([o]) \leq P_{\mathcal{M}}([o]) + \sqrt{\mathrm{CI}_{\mathcal{M}}(O; \varepsilon_h)/2}$. Combining with the mixture model (Assumption~\ref{assump:mixture}) yields the stated bound.

\subsection{Lemma~\ref{lem:dilution} (Equilibrium Epistemic Dilution) Full Proof}

Extends the pairwise Agent-to-Agent Dilution Corollary to the full cycle equilibrium. Let $C = (i_1 \to i_2 \to \cdots \to i_l \to i_1)$ be a mixed cycle containing at least one artificial agent $m$ with compound autonomy $\alpha^X(m) \cdot \alpha^E(m) \geq \hat{\Lambda}$. The cycle-emergent outcome type $[o^*]$ is determined by the equilibrium action profile $(A_{i_1}^*, \ldots, A_{i_l}^*)$, the fixed point of the composed structural equations around $C$ (existence and uniqueness guaranteed by contraction; Bongers et al.~\cite{bongers2021}, Theorem~3.1).

For any agent $i \in C$, the epistemic access $\hat{\kappa}_i([o^*])$ requires predicting the equilibrium actions of all other agents. Under Assumption~\ref{assump:mixture}, each artificial agent $m_j \in C$ generates actions according to $\omega_{m_j}(\cdot \mid s) = (1 - \alpha^X(m_j)) \cdot \pi_h(\cdot \mid s) + \alpha^X(m_j) \cdot \pi_g(\cdot \mid s, Z_{m_j}, \varepsilon_{m_j})$. Agent $i$'s prediction of $m_j$'s autonomous component is bounded by $(1 - \alpha^E(m_j))$ (from the definition of epistemic autonomy). Combining: agent $i$'s accuracy in predicting $m_j$'s action is bounded by $(1 - \alpha^X(m_j)) + \alpha^X(m_j)(1 - \alpha^E(m_j)) = 1 - \alpha^X(m_j) \cdot \alpha^E(m_j)$.

Applying the Data Processing Inequality to the causal chain $\varepsilon_i \to \text{prediction of } (A_{i_1}^*, \ldots, A_{i_l}^*) \to [o^*]$: the weakest link determines the bound. Since $C$ is a cycle, every agent $i$'s predictive chain passes through at least one artificial agent with compound autonomy $\geq \hat{\Lambda}$. Therefore $\hat{\kappa}_i([o^*]) \leq 1 - \hat{\Lambda}$ for all $i \in C$.

The critical distinction from the pairwise corollary is that in a cycle, artificial agents face the same bound: even $m_j$'s own prediction of $[o^*]$ depends on other agents' equilibrium responses to $m_j$'s output, and those responses involve autonomous components that $m_j$ cannot predict.

\subsection{Lemma~\ref{lem:nonadd} (Causal Non-Additivity) Full Proof}

Explicit computation of $\zeta$ for a linear 3-agent cycle: structural equations $A_{m_1} = \beta_1 A_h + \gamma_1 Z_1$, $A_{m_2} = \beta_2 A_{m_1} + \gamma_2 Z_2$, $A_h = \beta_3 A_{m_2} + \gamma_3 Z_h$, with equilibrium solution via matrix inversion. The interaction residual is $\zeta = \beta_1 \beta_2 \beta_3 / (1 - \beta_1 \beta_2 \beta_3) > 0$ when all $\beta_i > 0$. Continuity argument extends the result to non-linear equations in a neighbourhood of the linear case.

\subsection{Existence of Cycle-Emergent Outcome Types}

For any mixed cycle $C$ in a HAC satisfying Assumption~\ref{assump:mixture}, the equilibrium action profile of $C$ (the fixed point of the composed structural equations around $C$) constitutes a cycle-emergent outcome type. Define $[o^*]$ as the equivalence class of outcomes determined by the joint equilibrium actions of agents in $C$. By construction, $[o^*]$ depends only on the structural equations of agents in $C$ and the exogenous noise terms $U_i$ for $i \in C$. For any agent $j \notin C$, the do-intervention $\mathrm{do}(A_j = a_j)$ does not alter the structural equations of agents in $C$, so $P_{\mathcal{M}}([o^*] \mid \mathrm{do}(A_j = a_j)) = P_{\mathcal{M}}([o^*] \mid \mathrm{do}(A_j = a_j^0))$ and $\mathrm{CE}_{\mathcal{M}}(j, [o^*]) = 0$. Therefore $[o^*]$ is cycle-emergent with respect to $C$.

\subsection{Lemma~\ref{lem:bound} (Autonomy-Accountability Bound) Full Derivation}

Each human $h$'s responsibility satisfies $\rho_h([o]) \leq \hat{\kappa}_h([o])$ (Axiom~\ref{ax:foresee}, Eq.~\eqref{eq:ax2}). By Lemma~\ref{lem:dilution} (Eq.~\eqref{eq:dilution}), $\hat{\kappa}_h([o]) \leq 1 - \hat{\Lambda}$. Summing over $n$ humans: $\sum_{h \in H} \rho_h([o]) \leq n(1 - \hat{\Lambda})$.

\subsection{Interaction Residual Lower Bound}

Derives $\zeta([o^*]) \geq \hat{\Lambda}^2 \cdot C(G)/(1 + C(G))$ from the equilibrium structure of cyclic SCMs. This bound supports Corollary~\ref{cor:residual} by providing a quantitative causal measure of the governance gap complementing the epistemic mechanism of the main proof.

\subsection{Robustness Under Collective Proportionality}

If Axiom~\ref{ax:foresee} is replaced by a collective variant bounding the group's total responsibility by collective epistemic access $K(N, [o]) = \max_i \hat{\kappa}_i([o]) + \psi(N)$, the impossibility survives for $\hat{\Lambda}$ exceeding a higher threshold $\hat{\Lambda}_2^* > \hat{\Lambda}^*$, since the interaction residual $\zeta$ is a causal phenomenon that cannot be resolved by epistemic pooling.

\section{Invariance Results}

\subsection{Weight-Vector Invariance}

The Accountability Horizon $\hat{\Lambda}^* = 1 - 1/|C_{\min}|$ depends on $\hat{\Lambda} = \min_i\,[\alpha^X(m_i) \cdot \alpha^E(m_i)]$, which is defined directly from agent autonomy profiles without reference to the weight vector $\mathbf{w}$. Different $\mathbf{w}$ values change the summary statistic $\Lambda(\mathcal{H})$ (Eq.~\eqref{eq:CAI}) but not the threshold variable $\hat{\Lambda}$ or the feasibility classification.

\subsection{Centrality Invariance}

Theorem~\ref{thm:main} holds under any centrality measure $\eta$ satisfying (C1) and (C2). The proof proceeds by showing that $\Lambda(\mathcal{H})$ under alternative $\eta$ is bounded between $\Lambda(\mathcal{H})$ under out-degree and a linear transformation thereof. Since the impossibility depends on $\hat{\Lambda}$ (not $\Lambda$), the result is invariant.

\section{Sensitivity Analysis}

\subsection{$\tau$-Parameterised Accountability Horizon}

The combined threshold incorporating both the structural bound and the non-vacuity constraint is $\hat{\Lambda}^*(\tau) = \min(1 - 1/|C_{\min}|,\; 1 - \tau)$ (Eq.~\eqref{eq:combined-threshold}). For $\tau \leq 1/|C_{\min}|$, the structural bound dominates. For $\tau > 1/|C_{\min}|$, the non-vacuity constraint binds first.

\subsection{Measurement Error}

If autonomy estimates have error bounded by $\delta$, then $|\hat{\hat{\Lambda}} - \hat{\Lambda}| \leq 2\delta$ (by the product rule for error propagation). The Accountability Horizon $\hat{\Lambda}^* = 1 - 1/|C_{\min}|$ is measurement-invariant (depends only on graph topology). Classification margin: HACs within $2\delta$ of $\hat{\Lambda}^*$ cannot be confidently classified. Recommended safety margin: classify a HAC as above the Accountability Horizon whenever $\hat{\hat{\Lambda}} > \hat{\Lambda}^* - 2\delta$.

\subsection{$\delta$-Robustness of the Mixture-Model Assumption}

If agent $m$'s policy $\omega_m$ satisfies $\mathrm{TV}(\omega_m, \tilde{\omega}_m) \leq \delta$ where $\tilde{\omega}_m$ satisfies Assumption~\ref{assump:mixture}, then the epistemic dilution bound of Lemma~\ref{lem:dilution} degrades by at most $\delta$: $\hat{\kappa}_h([o]) \leq (1 - \hat{\Lambda}) + \delta$. The Accountability Horizon under the approximate model remains $\hat{\Lambda}^* = 1 - 1/|C_{\min}|$ (unchanged), but the system is confidently above $\hat{\Lambda}^*$ only when $\hat{\Lambda} > \hat{\Lambda}^* + \delta$. Combining with measurement error (Section~C.2), the total safety margin is $\hat{\Lambda}^* - 2\delta_{\mathrm{meas}} - \delta_{\mathrm{model}}$.

\section{Notation Table}

{\renewcommand\arraystretch{1.15}
\noindent\begin{longtable*}{@{}l @{\quad} l@{}}
\toprule
\textbf{Symbol} & \textbf{Definition} \\
\midrule
$H$, $M$, $N$ & Sets of human agents, artificial agents, all agents ($N = H \cup M$) \\
$n$, $k$ & Number of human agents, number of artificial agents \\
$E = (\Omega, O, \mathcal{T})$ & Environment: observation space, outcome space, outcome-generation function $\mathcal{T} : \Sigma \times \Omega_{\mathrm{ext}} \to \Delta(O)$ \\
$\mathcal{M}$ & Collective structural causal model \\
$G = (N, \mathcal{E})$ & Directed interaction graph \\
$\mathcal{H}$ & Human-Agent Collective \\
$\alpha^E$, $\alpha^X$, $\alpha^D$, $\alpha^S$ & Epistemic, executive, evaluative, social autonomy \\ $\theta_0$ & Approval-set threshold in Eq.~\eqref{eq:alpha-X}: $\theta_0 = 1/(2|A|)$ \\
$A(a) = \mathbf{w}^\top \alpha(a)$ & Aggregate autonomy \\
$\alpha^X \cdot \alpha^E$ & Compound autonomy (product of executive and epistemic) \\
$\hat{\Lambda}$ & Minimum compound autonomy: $\min_{i \in M^*}\,[\alpha^X(m_i) \cdot \alpha^E(m_i)]$ \\
$\hat{\Lambda}^*$ & Accountability Horizon: $1 - 1/|C_{\min}|$ \\
$|C_{\min}|$ & Number of agents in the smallest mixed feedback cycle \\
$C(G)$ & Number of mixed cycles in $G$ \\
$\Lambda(\mathcal{H})$ & Collective Autonomy Index \\
$\kappa_i(o)$ & Epistemic access of agent $i$ to outcome $o$ \\
$\hat{\kappa}_i([o])$ & Type-level epistemic access \\
$\rho : [O] \to \Delta(N)$ & Outcome attribution (responsibility shares) \\
$\mathrm{CE}_{\mathcal{M}}(i, [o])$ & Causal effect of agent $i$ on outcome type $[o]$ \\
$\zeta([o^*])$ & Interaction residual \\
$\Delta$ & Accountability deficit: $1 - |C_{\min}|(1 - \hat{\Lambda})$ \\
$\mathcal{R}(\mathcal{H})$ & Accountability residual: $1 - |C_{\min}|(1-\hat{\Lambda})$ (equals deficit $\Delta$; see Eq.~\eqref{eq:residual}) \\
$\tau$ & Non-Vacuity governance threshold \\
$\gamma$ & Partial allocation parameter (relaxation of Completeness) \\
$\mathcal{L}(\mathcal{H})$ & Set of legitimate accountability frameworks \\
$\pi_h$, $\pi_g$ & Human-aligned policy, autonomous policy \\
$\omega_m$ & Agent $m$'s action-generation policy \\
$\varepsilon_i$ & Epistemic state of agent $i$ \\
$\beta(m)$ & Information-autonomy coefficient \\
$\eta(m, G)$ & Normalised influence centrality \\
\bottomrule
\end{longtable*}}

\section{Supplementary Tables}

\begin{itemize}
    \item $\alpha^{D}: \text{mean} = 0.65, \sigma = 0.06$
    \item $\alpha^{S}: \text{mean} = 0.85, \sigma = 0.05$
    \item Min compound $\alpha^{X} \cdot \alpha^{E} = 0.621$
\end{itemize}

\begin{table}[h]
\centering
\caption{$\mathcal{H}_2$: Multi-Agent Financial Trading (Table 2 Expanded)}
\begin{tabular}{ll}
\hline
\textbf{Parameter} & \textbf{Value} \\ \hline
Humans ($n$) & 10 \\
Machines ($k$) & 8 \\
Mixed cycles $C(G)$ & 77 \\
$|C_{min}|$ & 3 \\
$\Lambda(\mathcal{H})$ & 0.7925 \\
$\hat{\Lambda}$ (min compound) & 0.6210 \\
$\hat{\Lambda}^*$ (Horizon) & 0.6667 \\
$\hat{\Lambda} > \hat{\Lambda}^*$? & False \\
Epistemic Budget & 1.1370 \\
Residual $\Delta$ & 0.0000 \\ \hline
\textbf{Status} & \textbf{Below Horizon (Governable)} \\ \hline
\end{tabular}
\end{table}
  
\end{document}